\documentclass[opre,nonblindrev]{informs3}

\OneAndAHalfSpacedXI

\usepackage{natbib}
\usepackage{algorithm}
\usepackage{subfig}
\usepackage[noend]{algpseudocode}
\usepackage{longtable, multirow}
\usepackage{graphicx}

\makeatother

 \bibpunct[, ]{(}{)}{,}{a}{}{,}%
 \def\bibfont{\small}%
 \def\bibsep{\smallskipamount}%
 \def\bibhang{24pt}%

\newcommand{\E}{\mathbb{E}}
\newcommand{\prob}{\mathbb{P}}

\TheoremsNumberedThrough     
\ECRepeatTheorems

\EquationsNumberedThrough    

\MANUSCRIPTNO{OPRE-2018-05-302}

\begin{document}


\RUNAUTHOR{Li, Zhong, and Brandeau}

\RUNTITLE{Quantile Markov Decision Processes}

\TITLE{Quantile Markov Decision Processes}

\ARTICLEAUTHORS{%
\AUTHOR{Xiaocheng Li}
\AFF{Department of Management Science and Engineering, Stanford University, Stanford, CA, 94305, \EMAIL{chengli1@stanford.edu}} 
\AUTHOR{Huaiyang Zhong}
\AFF{Department of Management Science and Engineering, Stanford University, Stanford, CA, 94305, \EMAIL{hzhong34@stanford.ed}}
\AUTHOR{Margaret L. Brandeau}
\AFF{Department of Management Science and Engineering, Stanford University, Stanford, CA, 94305, \EMAIL{brandeau@stanford.edu}}
} 

\ABSTRACT{%
\textbf{Abstract.} The goal of a traditional Markov decision process (MDP) is to maximize expected cumulative reward over a defined horizon (possibly infinite). In many applications, however, a decision maker may be interested in optimizing a specific quantile of the cumulative reward instead of its expectation. In this paper we consider the problem of optimizing the quantiles of the cumulative rewards of a Markov decision process (MDP), which we refer to as a quantile Markov decision process (QMDP). We provide analytical results characterizing the optimal QMDP value function and present a dynamic programming-based algorithm to solve for the optimal policy. The algorithm also extends to the MDP problem with a conditional value-at-risk (CVaR) objective. We illustrate the practical relevance of our model by evaluating it on an HIV treatment initiation problem, where patients aim to balance the potential benefits and risks of the treatment.
}%


 \HISTORY{}

\maketitle

%


\section{Introduction}
The problem of sequential decision making has been widely studied in the fields of operations research, management science, artificial intelligence, and stochastic control. Markov decision processes (MDPs) are one important framework for addressing such problems. In the traditional MDP setting, an agent sequentially performs actions based on information about the current state and then obtains rewards based on the action and state. The goal of an MDP is to maximize the expected cumulative reward over a defined horizon which may be finite or infinite. 

In many applications, however, a decision maker may be interested in optimizing a specific quantile of the cumulative reward instead of its expectation. For example, a physician may want to determine the optimal drug regime for a risk-averse patient with the objective of maximizing the $0.10$ quantile of the cumulative reward; this is the cumulative improvement in health that is expected to occur with at least 90\% probability for the patient \citep{beyerlein2014quantile}. A company such as Amazon that provides cloud computing services might want their cloud service to be optimized at the $0.01$ quantile  \citep{DeCandia:2007:DAH:1323293.1294281}, meaning that the company strives to provide service that satisfies $99\%$ of its customers. In the finance industry, the quantile measure, sometimes referred as value at risk (VaR), has been used as a measure of capital adequacy \citep{Duffie97}. For example, the 1996 Market Risk Amendment to the Basel Accord officially used VaR for determining the market risk capital requirement \citep{Berkowitz2002}.  The advantage of a quantile objective lies in its focus on the distribution of rewards. The cumulative reward usually cannot be characterized by the expectation alone; distributions of cumulative reward with the same expectation may differ greatly in their lower or upper quantiles, especially when they are skewed. 

In this paper, we study the problem of optimizing quantiles of cumulative rewards of a Markov decision process, which we refer to as a quantile Markov decision process (QMDP). Our QMDP formulation considers a quantile objective for an underlying MDP with finite states and actions, and with either a finite or infinite horizon. We show that the key to solving the optimal policy for a QMDP problem is proper augmentation of the state space. The augmented state acts as a performance measure of the past cumulative reward and thus ``Markovizes" the optimal decisions. This enables us to develop an efficient dynamic programming procedure that finds the optimal QMDP value function for all states and quantiles in one pass. In the execution of the optimal policy,  the augmented state guides the strategy in subsequent periods to be aggressive, neutral, or conservative. We also demonstrate how the same idea extends to the conditional value at risk (CVaR) objective. 

\subsection{Main Contributions}

In this section we describe the contribution of our work in three areas: model formulation (as a risk-sensitive MDP), solution methodology (the design of a dynamic programming algorithm to handle a non-Markovian objective), and practical application.

\textbf{Risk-sensitive MDP.} 
There are two types of uncertainty associated with an MDP: \textit{inherent uncertainty}, which is the variability of cumulative cost/reward caused by the stochasticity of the MDP itself, and \textit{model uncertainty}, which is the uncertainty caused by unavoidable model ambiguity or parameter estimation errors \citep{delage2010percentile, doi:10.1287/moor.1120.0566}. The QMDP model aims to explicitly characterize the inherent uncertainty of an MDP by looking at the quantiles and the distribution of the cumulative reward.

The study of risk-sensitive MDPs dates back to \cite{howard1972risk} who proposed the use of an exponential utility function to capture risk attitude. The authors developed a policy iteration algorithm that relies on the structure of the exponential function to solve for the optimal policy. Subsequently, \cite{piunovskiy2006dynamic, tamar2012policy}, and \cite{mannor2011mean} imposed a variance constraint, and \cite{altman1999constrained} and \cite{ermon2012probabilistic} added a probabilistic constraint on the MDP cumulative reward. The variance or probability constraint characterizes the variation of the cumulative reward and is introduced to control the internal risk of MDP. \cite{di2012policy} derived policy gradient algorithms for variance-related risk criteria and
\cite{arlotto2014markov} identified types of MDP problems where the mean of the cumulative reward dominates its variance. \cite{chow2017} noted that the optimal policy for such models is usually very sensitive to the choice of the risk parameter value and to misspecification of the underlying probability distribution. 

Another stream of research on risk-sensitive MDPs employs risk measures that account for the variation of the cumulative reward. \cite{ruszczynski2010risk} and \cite{jiang2017risk} proposed nested risk measures for a risk-averse MDP problem. The nested objective inductively summarizes the cost-to-go reward at each time step into a deterministic value; thus the problem can be solved by a dynamic programming procedure similar to that for a traditional MDP. One shortcoming of the nested risk measure is that there is no clear relation between the cumulative reward and the optimal nested objective function value. Moreover, the nested risk measure involves a user-specified parameter. We will further elaborate on the difference between QMDP and nested risk measure models through an example in Section \ref{illustration}.

The risk-sensitive MDP models described in this section solve the optimal policy for only one risk parameter at a time. Consequently, they require prior knowledge to specify the risk parameter; if, after obtaining the optimal policy for a given parameter value, the reward is not satisfactory, one has to solve the model again with another parameter value, essentially using a trial-and-error procedure. In contrast, QMDP outputs the optimally achievable quantile of the cumulative reward for all quantiles (the risk parameter in the QMDP model) in a single run of dynamic programming.   

\vspace{0.2cm}

\noindent \textbf{Dynamic programming for a non-Markovian objective.} The main difficulty in solving risk-sensitive MDP models is the design of an efficient dynamic programming algorithm. 
Nested risk measure models \citep{ruszczynski2010risk, jiang2017risk} compose a sequence of one-step risk measures. Since the optimal action in each period depends only on the current state, these models avoid the inconvenience of dealing with non-Markovian structures. Certain model structures can be utilized for algorithm design, such as the dual-based dynamic programming approach for the multi-stage stochastic programming problem \citep{Shapiro2013}.  For MDP with the CVaR objective, a number of studies \citep{Bauerle2011, Yu2017, chow2014algorithms, chow2015risk} have each solved the problem under slightly different settings. A common technique employed in these papers is augmentation of the state space and execution of a dynamic programming algorithm in the augmented state space; however, the augmentation methods are restricted to the CVaR objective and cannot be generalized to handle the quantile objective. 

QMDP deals with a non-Markovian objective where the optimal policy may depend on the entire past history. Our solution algorithm provides a state-augmentation method to handle the non-Markovian objective and complements the literature on dynamic programming algorithms. The augmented state for the quantile objective acts like a ``sufficient statistic'' for the past history. The dynamic programming algorithm is executed over the augmented state space with an optimization subroutine. Compared to the augmented state in the CVaR MDP, our augmented variable conveys a tangible meaning -- quantile -- whereas the augmented state in \cite{Bauerle2011} and \cite{Yu2017} is only a nominal variable to facilitate the solving of the optimization problem. The QMDP cost-to-go function at each time period is a function of the current state and the augmented state (quantile), and represents the optimal value function of a QMDP subproblem for the remaining periods (given the current state and for all quantiles), in contrast to the nested risk measure formulation \citep{ruszczynski2010risk, jiang2017risk} where the cost-to-go function is simply a deterministic value. This special property of the augmented state enables us to solve QMDP for the optimal value function and the optimal policy for all quantiles in one pass of dynamic programming. The other formulations can only solve one risk parameter at a time, and the CVaR MDP algorithms proposed by \cite{Bauerle2011} and \cite{Yu2017} could possibly require solving dynamic programming procedures infinitely many times to obtain the optimal value function and policy for a single percentile parameter.  

The dynamic programming results from QMDP also provide insights for understanding a dynamic quantile risk measure and give a non-constructive explanation for the time-inconsistency \citep{cheridito2009time} of the quantile risk measure. A quantile objective specifies a family of risk measures. The execution of the QMDP solution procedure entails a dynamic change of the risk measure within the family. This makes the conventional definition of dynamic risk measure unsuitable for the quantile objective. In Section \ref{riskmeasure}, we discuss this issue in detail and develop a new notion of time-consistent risk measure. 

\vspace{0.2cm}

\noindent \textbf{Practical Relevance of QMDP.}
MDP models have been widely applied to many real-world problems  including, for example, financial derivative pricing \citep{Tsitsiklis1999},  service system planning \citep{sennott_1989}, and chronic disease treatment \citep{shechter2008optimal, mason2014optimizing}. However, these applications do not consider the fact that many decisions are inherently risk-sensitive. For instance, both physicians and patients are concerned about the risk associated with different medical treatment decisions. Practitioners have applied the quantile objective in a variety of applications, but in a descriptive manner \citep{Berkowitz2002,austin2005use, beyerlein2014quantile}. Our work contributes to the adoption of quantile criteria in sequential decision making. In prior work there has been no clear solution to decode the full distribution of cumulative reward. With a single pass of the QMDP solution algorithm, we can obtain the optimal rewards for all quantiles. Comparing these rewards to the quantiles of the cumulative reward under the traditional optimal MDP policy can help assess the need for adoption of a risk-sensitive framework. In this sense, QMDP is not a substitute for but a complement to MDP models in real-world applications. 

\subsection{Other Related Literature}

To the best of our knowledge, this paper is the first to address the MDP problem with a quantile objective in a generic setting. Several studies have examined restricted versions of the problem. \cite{filar1995percentile} studied the quantile objective for the limiting average reward of an infinite-horizon MDP, determining whether there exists a policy that achieves a specified value of the long-run limiting average reward at a specified probability level. \cite{ummels2013computing} developed an algorithm to compute the quantile cumulative rewards for a given policy in polynomial time. The algorithm is descriptive rather than prescriptive in terms of understanding the uncertainty associated with the Markov chain (MDP with a fixed policy). \cite{Gilbert2016} addressed the quantile MDP problem for the special case of deterministic rewards and preference-based MDP.   

CVaR, also known as average value at risk (AVaR) or expected shortfall, has been explored in the context of risk-sensitive MDPs. CVaR is defined as the expectation of the cost/reward in the worst $q$\% of cases. From the perspective of chance constrained optimization, the CVaR criterion can be viewed as a convex relaxation of the quantile criterion and thus can be optimized more conveniently \citep{nemirovski2006convex}. \cite{Bauerle2011} utilized a variational representation of the CVaR criteria and derived an analytical framework for solving MDP with a CVaR objective. The variational form expresses the optimal value of CVaR MDP as an optimization of a univariate function on the real line. The function value at each real number must be computed by executing a dynamic programming procedure in the same way as for a traditional MDP. \cite{Yu2017} studied MDP under a more general class of risk measures that have similar variational form. The algorithms in \cite{Bauerle2011} and \cite{Yu2017} optimize a function via grid search and employ a dynamic programming subroutine for the underlying MDP and thus offer no complexity guarantee and only solve for a single percentile each time. In contrast, QMDP solves for all the quantiles in a single pass of dynamic programming. 

Recent work has explored the interaction of the CVaR objective with MDP. \cite{Carpin} studied the CVaR objective for the total cost/reward of transient MDPs. \cite{chow2014algorithms} considered MDP with an expectation objective and a CVaR constraint. \cite{chow2015risk} considered the CVaR objective for the cumulative reward, which is close to the quantile objective in this paper, but only considered infinite-horizon discounted MDPs. In this paper, we show that the derivation of our QMDP model naturally extends to CVaR MDP, and we provide an exact algorithm for solving the CVaR MDP problem (including for the case of a finite horizon and an undiscounted setting not considered by \cite{chow2015risk}). 

Finally, in the area of reinforcement learning \cite{bellemare2017distributional} proposed a distributional perspective and derived a method that outputs the distribution, rather than just the expectation, of the optimal cumulative reward. The optimal policy was defined as maximizing the expectation of the cumulative reward. Their method provides additional insights regarding the distribution of the optimal reward. Subsequent studies \citep{dabney2018distributional, yang2019fully} examined different ways to parameterize and learn the distributional value function. These studies, though still adopting expectation as the optimality criterion, shed light on the importance of the distributional information in a sequential decision making context.

\subsection{Illustration of the QMDP Output}

The method presented in this paper computes the QMDP optimal value function and optimal policy for all quantiles with a single pass of dynamic programming. Figure \ref{fig:opt_val_illu} shows the QMDP optimal value function for three different underlying MDPs that share the same state and action space but have different reward functions and transition probabilities (this example is worked out in Section 6.1). Each point on the red solid curve indicates the optimally achievable quantile level for the cumulative reward. The gray dashed curve shows the empirical cumulative density function (CDF)  of the cumulative reward under the traditional MDP optimal policy that maximizes the expected reward. 


\begin{figure}[t]
\centering
\includegraphics[height=2.5 in]{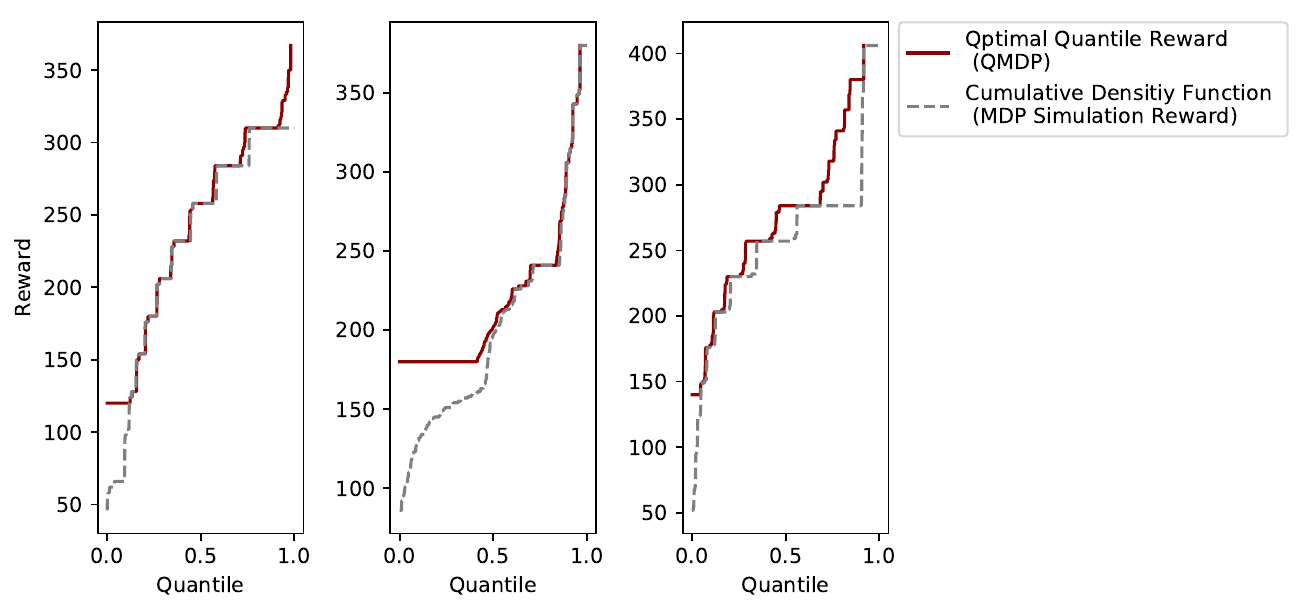}
\caption{\small{Comparison of MDP and QMDP value functions. Each plot is obtained from a different initialization of the model parameters. The red lines are the optimal quantile rewards computed via QMDP. The gray dashed lines are the cumulative density function for simulations with the execution of the optimal MDP policy.}}\centering
\label{fig:opt_val_illu}
\end{figure}

The QMDP optimal value function captures the distributional information of the cumulative reward. For nested risk measure or utility function formulations the value function could be a complicated nonlinear transformation of the cumulative reward; if we want to know the distribution of the cumulative reward, simulation of each policy is necessary (in the same way that the empirical CDF is obtained for the traditional MDP). Additionally, risk-sensitive MDP models usually involve a trade-off procedure between risk and reward, which entails solving models for multiple parameters. For QMDP the optimal value function can be computed for all parameters at once. 

The QMDP model output provides a risk assessment for the underlying MDP problem. The three MDP problems in Figure \ref{fig:opt_val_illu} have different patterns of inherent risk. For the example in the middle panel, if a quantile reward is desired, there is an opportunity for significant improvement for quantiles below the median. In this case, a risk-sensitive MDP model might be desirable for a risk-averse decision maker. For the example in the left panel, the only significant difference occurs at the lowest quantiles. In this case the optimal policy under the traditional MDP, although not necessarily achieving quantile optimality, is quite stable and robust. In this way, QMDP can be used to determine whether a risk-sensitive MDP model is desirable and what kind of improvement one would expect if adopting a risk-sensitive MDP. 

\smallskip

The remainder of this paper is organized as follows. We lay out the traditional MDP problem formulation and assumptions in Section 2 and  present the QMDP problem formulation and dynamic programming solution in Section 3. We describe the algorithm for solving QMDP as well as its computational aspects in Section 4. We discuss extensions of the model in Section 5. We present empirical results on a synthetic example as well as on an HIV treatment initiation problem in Section 6. We conclude with discussion in Section 7.

\section{Markov Decision Process}

A Markov decision process (MDP) consists of two parts \citep{bertsekas1995dynamic}: an underlying discrete-time dynamic system, and a reward function that is additive over time. A dynamic system defines the evolution of the state over time:
\begin{equation}
S_{t+1} = f_t(S_t, a_t, w_t), \ \ t=0,1,...,T-1,
\label{dynamics}
\end{equation}
where $S_t$ denotes the state variable at time $t$ from state space $\mathcal{S}$, $a_t$ denotes the actions/decisions at time $t$ and $w_t$ is a random variable that captures the stochasticity in the system.  The reward function at time $t$, denoted by $r_t(S_t, a_t, w_t)$, accumulates over time. The cumulative reward is  
$$r_T(S_T) + \sum_{t=0}^{T-1}r_t(S_t, a_t, w_t),$$
where $r_{T}(S_T)$ is the terminal reward at the end of the process. The random variable $w_t \in \mathcal{W}$ determines the transition in the state space and the state $S_{t+1}$ follows a distribution $P_t(\cdot|S_t,a_t)$ that is possibly dependent on the state $S_t$ and the action $a_t.$ We consider the class of policies that consist of a sequence of functions $\pi = \{\mu_0,...,\mu_{T-1}\}$ where $\mu_t$ maps historical information $\mathcal{H}_{t}=(S_0,a_0,...,S_{t-1},a_{t-1},S_t)$ to an admissible action $a_t\in \mathcal{A}_t \subset \mathcal{A}.$ Here we use $\mathcal{A}_t$ and $\mathcal{A}$ to denote the admissible action set. The policy $\pi$ together with the function $f_t$ determines the dynamics of the process. Given an initial state $S_0$ and a policy $\pi$, we have the following expected total reward:
$$\E^{\pi} \left[r_T\left(S_T\right) + \sum_{t=0}^{T-1}r_t\left(S_t, a_t, w_t\right)\right].$$

The objective of an MDP is to choose an optimal policy in the set $\Pi$ of all admissible policies that maximizes the expected total reward, i.e.
\begin{equation}
\max_{\pi \in \Pi} \E^{\pi} \left[r_T\left(S_T\right) + \sum_{t=0}^{T-1}r_t\left(S_t, a_t, w_t\right)\right],
\label{expectMDP}
\end{equation}
where the expectation is taken with respect to $(w_0, w_1,...,w_{T-1}).$ Without loss of generality, we assume $r_T\left(S_T\right)=0$ for all $S_T.$

\subsection{Assumptions}

We first discuss a few assumptions and clarify the scope of this paper. 

\begin{assumption}[State and Action Space]
\begin{itemize}
    \item[]
    \item [(a)] The state space $\mathcal{S}$ and the action space $\mathcal{A}$ are finite.
    \item[(b)] The random variable $w_t \in \mathcal{W}$ has a finite support, i.e. $|\mathcal{W}| < \infty.$
    \item[(c)] The function $f_t$ is ``weakly invertibile'': $\mathcal{S}\times \mathcal{A} \times \mathcal{W} \rightarrow \mathcal{S}$ governs the dynamic system (\ref{dynamics}). Specifically, there exists a function $l_t: \mathcal{S}\times \mathcal{A} \times \mathcal{S} \rightarrow \mathcal{W}$ such that for any $s\in \mathcal{S},$ $a \in \mathcal{A}$ and $w\in \mathcal{W},$
$$l_t(s,a,f_t(s,a,w)) = w.$$
\end{itemize}
\label{assume}
\end{assumption}

Part (a) is a classic assumption about the finiteness of the state and action spaces. In part (b), we assume that the random variable $w_t$ has a finite support, i.e. $w_t$ is a discrete random variable only taking finite possible values. This paper concerns the quantiles of cumulative reward; if $w_t$ has infinite support, it will result in the reward $r_t\left(S_t, a_t, w_t\right)$ and the cumulative reward having infinite support. In fact, there is no general way to store the infinite support random variable or to query its quantiles unless the distribution has some parameterized structure. Since we aim for a generic treatment of the quantile MDP problem, the assumption of finite support is necessary. Also, because a random variable can be always approximated by a finitely supported discrete random variable at any granularity, we believe part (b) will not cause much practical limitation.

Part (c) is introduced for notational simplicity in our derivation. We show how to remove this assumption in Appendix A. Part (c) states that the random variable $w_t$ can be fully recovered with the knowledge of $S_t, a_t$, and $S_{t+1},$ i.e. there exists a function $l_t$ s.t. $w_t = l_t\left(S_t, a_t, S_{t+1}\right).$ This assumption means that there is no additional randomness other than that which governs the state transitions. It follows that the reward $r_t$ will be a function of $S_{t},$ $a_t$, and $S_{t+1}.$ In practice, this assumption is well satisfied by most MDP applications. Additionally, we allow the dynamics $f_t(\cdot)$ in part (c) and the reward function $r_t(\cdot)$ to be non-stationary and non-parametric.

\section{Quantile Markov Decision Process}

In this section, we formulate the QMDP problem and derive our main result -- a dynamic programming procedure to solve QMDP. All proofs are provided in Appendix B.

\subsection{Quantile Objective and Assumptions}
\label{formulateQMDP}

The quantile of a random variable is defined as follows.
\begin{definition}
\label{qt}
For $\tau \in (0,1),$ the $\tau$-quantile of a random variable $X$ is defined as 
$$Q_\tau(X) \triangleq \inf\{x\mid \prob(X\le x) \ge \tau \}.$$
For $\tau=0,1$  we define $Q_0(X) = \inf\{X\}$ and $Q_1(X) = \sup\{X\}$, respectively.\footnote{Here we do not consider the effect of $0$-measure set. More precisely, the definition should be $Q_0(X) = \sup\{D\in \mathbb{R}|P(X\ge D)=1\}$ and $ Q_1(X) = \inf\{U\in \mathbb{R}|P(X\le U)=1\}.$
}
\end{definition}

The following properties are implied by the definition.
\begin{lemma}
For a given random variable $X$, $Q_\tau(X)$ is a left continuous and non-decreasing function of $\tau.$ Additionally, 
$$\prob\left(X\le  Q_\tau(X)\right) \ge \tau.$$
 \label{quantInq}
\end{lemma}

The goal of the QMDP is to maximize the $\tau$-quantile of the total reward:
\begin{equation}
\max_{\pi \in \Pi} Q^{\pi}_{\tau} \left[\sum_{t=0}^{T-1}r_t\left(S_t, a_t, w_t\right)\right].
\label{quantileMDP}
\end{equation}
Here the quantile is taken with respect to the random variables $(w_0, w_1, ..., w_{T-1})$, and the superscript $\pi$ denotes the policy we choose. The above formulation is for the case of a fixed finite horizon, i.e. $T<\infty.$ For the infinite-horizon case, the objective is
\begin{equation}
\max_{\pi \in \Pi} Q^{\pi}_{\tau} \left[\sum_{t=0}^{\infty} \gamma^t r_t\left(S_t, a_t, w_t\right)\right],
\end{equation}
where $\gamma\in(0,1)$ is the discount factor.

As in the derivation of MDP with expectation objective, we introduce a value function for the quantile reward of the Markov decision process. Suppose that the process initiates in state $s$ at time $t$, and we adopt the policy $\pi_{t:T}$. The value function is
$$v_t^{\pi_{t:T}}(s, \tau)   \triangleq  Q_\tau\left[\sum_{k=t}^{T-1}r_k\left(S_k, a_k, w_k\right) \Big|S_t =s \right].$$
Here $\pi_{t:T} = (\mu_{t},...,\mu_{T-1})$ denotes the policy and the action $$a_k = \mu_k(\mathcal{H}'_k) = \mu_{k}(S_t,a_t,...,S_k)$$ for $k=t,...,T-1.$ Since the process initiates at time $t$, the history $\mathcal{H}'_k$ also begins with $S_t$. We emphasize that the value function is a function of both the state $s$ and the quantile of interest $\tau$ and is indexed by time $t$.
The value function also depends on the chosen policy $\pi_{t:T}$.

The objective of QMDP is to maximize the value $v_t^{\pi_{t:T}}(s, \tau)$ by optimizing the policy $\pi_{t:T}$. Thus, we define the optimal value function as
\begin{equation}
v_t(s, \tau)  \triangleq  \max_{\pi_{t:T} \in \Pi} v_t^{\pi_{t:T}}(s, \tau).
\label{val}
\end{equation}
When $t=0,$ the value function $v_0(s, \tau)$ is equal to the optimal value in (\ref{quantileMDP}). 

\subsection{Value Function and Dynamic Programming}
\label{solveQMDP}

We construct a dynamic programming procedure and derive the optimal value function $v_t(s, \tau)$ backward from $t=T-1$ to $t=0$. The key step is to relate the value functions $v_t(s, \tau)$ with $v_{t+1}(s, \tau).$ Intuitively, $v_{t+1}(s, \tau)$ is obtained by optimizing $\pi_{(t+1):T}$ while $v_{t}(s, \tau)$ is obtained by optimizing $\pi_{t:T}.$ The difference lies in the choice of $\pi_t=\mu_t(\cdot).$ To connect them, we introduce an intermediate value function by fixing the output action of $\mu_t(s)$ to be $a$:
$$\widetilde{v}_t(s, \tau, a) \triangleq  \max_{\{\pi_{t:T} \in \Pi | \mu_t(s)=a\}} v_t^{\pi_{t:T}}(s,\tau).$$
Note that
\begin{equation}
v_t(s,\tau) = \max_{a} \widetilde{v}_t(s,\tau,a).
\label{v_s_a}
\end{equation}
We now establish the relationship between $\widetilde{v}_t(s,\tau,a)$ and the value function $v_{t+1}(s', \tau')$. We have
  \begin{align}
\widetilde{v}_t(s, \tau, a) =&  \max_{\{\pi_{t:T} \in \Pi | \mu_t(s)=a\}} v_t^{\pi_{t:T}}(s,\tau)\nonumber  \\ 
= & \max_{\{\pi_{t:T} \in \Pi | \mu_t(s)=a\}}  Q_\tau\left(\sum_{s' \in \mathcal{S}}  1\{S_{t+1}=s' | S_t=s, a_t=a\} \left[\sum_{k=t}^{T-1}r_k\left(S_k, a_k, w_k\right) \Big|S_t =s, S_{t+1} = s' \right]\right).
\label{actionValFunc}
\end{align}
Here the second line is obtained by differentiating possible values for the state $S_{t+1}.$ It is a summation of $|\mathcal{S}|$ random variables, each of which is associated with a specific state $s'.$ Analyzing each term more carefully, we have,
$$1\{S_{t+1}=s' | S_t=s, a_t=a\} \left[\sum_{k=t}^{T-1}r_k(S_k, a_k, w_k) \Big|S_t =s, S_{t+1} = s' \right]$$
$$=  1\{S_{t+1}=s' | S_t=s, a_t=a\} r_t(S_t, a_t, w_t) + 1\{S_{t+1}=s' | S_t=s, a_t=a\} \left[\sum_{k=t+1}^{T-1}r_k(S_k, a_k, w_k) \Big|S_t =s, S_{t+1} = s' \right].$$
The first term here is deterministic with the knowledge of $S_t$ and $S_{t+1}$ under Assumption \ref{assume} (c). The second term seems to be closely related to the value function $v_{t+1}(s',\tau')$ in that the summation begins from $t+1$ and the conditional part includes the information of $S_{t+1}.$
The following theorem formally establishes the relationship between $\widetilde{v}_t(s,\tau,a)$ and $v_{t+1}(s',\tau').$
\begin{theorem}[Value Function Dynamic Programming]
Let $\mathcal{S} = \{s_1,...,s_n\}$. Solving the value function defined in (\ref{val}) is equivalent to solving the following optimization problem:
  \begin{align}
 OPT\left(s,  \tau, a, v_{t+1}(\cdot,\cdot)\right) \ &  \triangleq\  \max_{\mathbf{q}} \min_{i \in \{ q_i \neq 1 \mid i= 1,2,...,n\}}  \ \   \left[v_{t+1}(s_i,q_i) + r_t\left(s, a, w_t\right)\right],
  \label{OPT} \\ 
  & = \ \max_{\mathbf{q}} \min_{i \in \{ q_i \neq 1 \mid i= 1,2,...,n\}}  \ \   \left[v_{t+1}(s_i,q_i) + r_t\left(s, a, l_t(s, a, s_i)\right)\right], \nonumber \\
  \text{subject to} & \ \  \sum_{i=1}^n p_i q_i \le \tau , q_i \in [0,1], \ p_i = \prob(S_{t+1}=s_i|S_t=s, a_t=a). \nonumber
 \end{align}
Here $w_t=l_t(s, a, s_{t+1})=l_t(s, a, s_i)$ is from Assumption \ref{assume} (c). We use $v_{t+1}(\cdot,\cdot)$ to denote the value function at $t+1$ and to emphasize that it is a function of state and  quantile. The decision variable here is the vector $\mathbf{q}$. Then,
$$ \widetilde{v}_t(s,\tau,a)  = OPT\left(s,  \tau, a, v_{t+1}(\cdot,\cdot)\right).$$
\label{valFunc}
\end{theorem}

The optimization problem stated in the theorem comes from the following lemma which computes the quantile of a sum of random variables (as it appears in the right-hand side of (\ref{actionValFunc})). Recall that the expectation of a summation of random variables equals the summation of the expectations, and this linearity makes possible the backward dynamic programming in the traditional MDP. Lemma \ref{XY} plays a similar role in that it relates the quantile of the summation of random variables to the quantiles of each random variable. This result together with the optimization algorithm in the next section is of potential interest for other applications concerned with the quantiles of random variables. 

\begin{lemma} 
\label{XY}
Consider $n$ discrete random variables $X_i$, $i=1,...,n,$ (here and hereafter, by discrete random variables, we mean that $X_i$ take values on a finite set) and another $n$ binary random variables $Y_i \in \{0, 1\}$ with $\sum_{i=1}^n Y_i=1.$ 
Then the quantile of the summation
$$ Q_\tau\left(\sum_{i=1}^n X_i Y_i \right)$$ is given by the solution to the following optimization problem:
\begin{align}
   \max_\mathbf{q} &\ \ \min_{i \in \{ q_i \neq 1 \mid i= 1,2,...,n\}}  \ \ Q_{q_i}(X_i) \label{QX_i} \\
\text{subject to} & \ \  \sum_{i=1}^n p_i q_i \le \tau,\nonumber \\
&q_i \in [0,1], \ p_i = \prob(Y_i=1). \nonumber 
\end{align}
Here $\mathbf{q}=(q_1,...,q_n)$ is the decision variable and $Q_{q_i}(X_i)$ is the $q_i$-quantile of the conditional distribution $X_i|Y_i=1.$
\end{lemma}

The key idea for the proof of Theorem \ref{valFunc} is to introduce a random variable $X_i$ such that its quantile $Q_{\tau}(X_i)=Q_{\tau}(v_{t+1}(s_i, \tau)+r_t(s,a,h(s,a,s_i)))$ for all $\tau \in [0,1].$ Then the right-hand side of (\ref{OPT}) is in the same form as (\ref{QX_i}) and Lemma \ref{XY} applies. By putting Theorem \ref{valFunc} together with (\ref{v_s_a}), we establish the relationship between $v_{t}(s,\tau)$ and $v_{t+1}(s',\tau')$ and build the foundation for a backward dynamic program to compute optimal value functions. Importantly, the algorithm derives the entire value function, i.e., the output we obtain at time $t$ is the function $v_t(\cdot,\cdot)$ rather than its evaluation at some specific $s$ and $\tau.$  

\subsection{Optimal Value and Optimal Policy}

\label{optimality}

In this section, we establish that the value functions obtained from the backward dynamic program correspond to the optimal value for the QMDP and thus define the optimal policy. The procedure for computing the value functions is illustrated in Figure \ref{fig:VALFUNC}. The optimization problem OPT takes $v_{t+1}$ as its argument and outputs $\tilde{v}_{t}(s,\tau,a)$; then by taking maximum over the action $a$, we obtain $v_{t}.$ Theorem \ref{value} verifies that the value function $v_0$ computed via backward dynamic programming is equal to the optimal quantile value.

\begin{figure}[t]
\centering
\includegraphics[height=2.8 in]{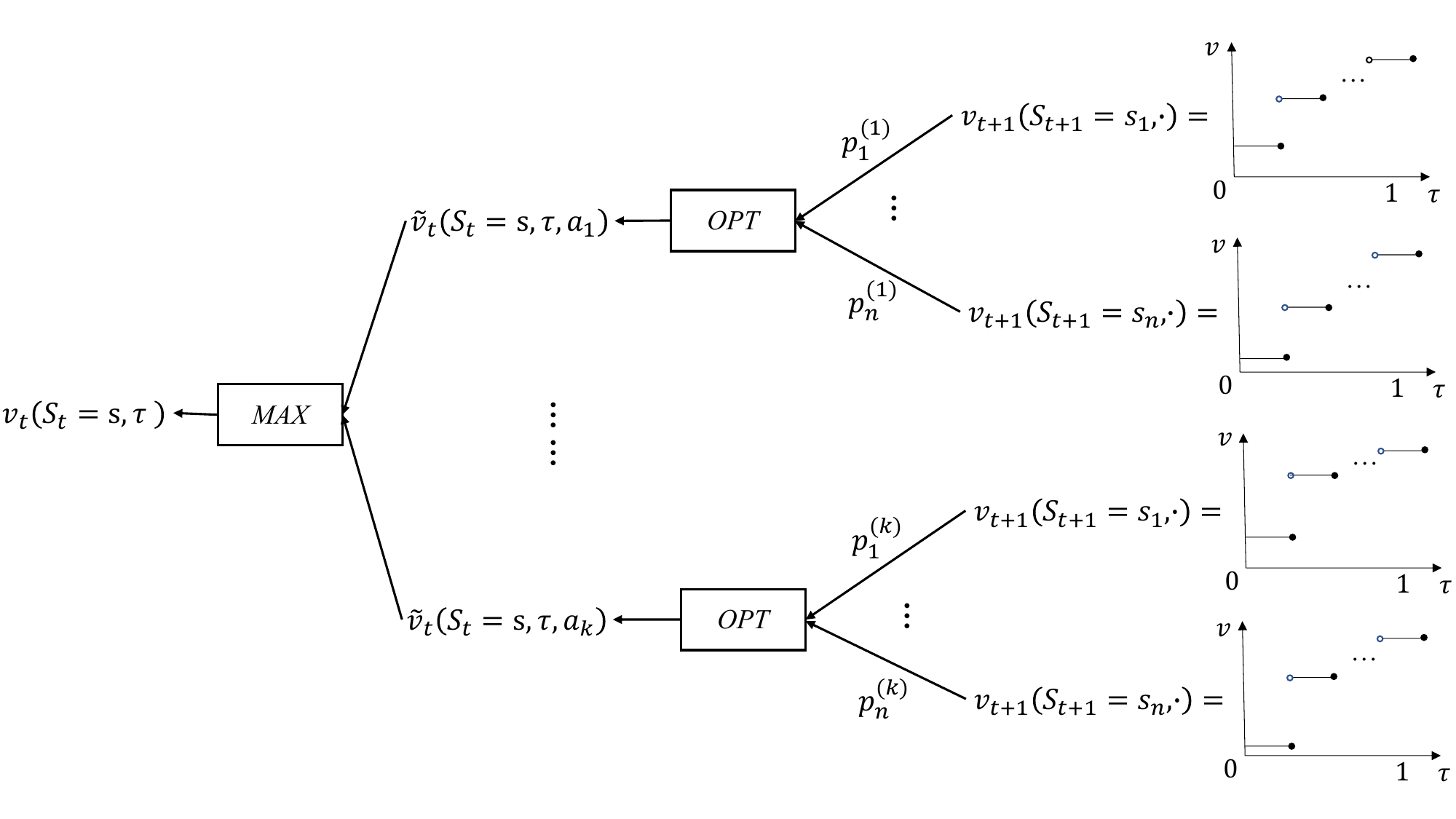}
\caption{Illustration of backward dynamic program for computing $v_t$ from $v_{t+1}$. Here $p_{i}^{(k)}=\prob(S_{t+1}=s_k|S_t=s, a=a_k).$ Without loss of generality, the immediate reward $r_t(S_t=s, a_k)$ is ignored in the schematic.}
\centering
\label{fig:VALFUNC}
\end{figure}

\begin{theorem}[Optimal Value Function] Let $v_T(s,\tau)=0$ for all $s\in \mathcal{S}$ and $\tau \in [0,1].$ Iteratively, we compute 
$$v_t(s,\tau) = \max_{a} \ OPT\left(s, \tau, a, v_{t+1}(\cdot, \cdot)\right),$$ 
for $t=T-1,...,0.$ Then we have
$$v_0(s,\tau) = \max_{\pi \in \Pi}\  Q^{\pi}_{\tau} \left[ \sum_{k=0}^{T-1}r_k(S_k, a_k, w_k)\right].$$
 \label{value}
\end{theorem}

Theorem \ref{policy} characterizes the optimal policy. Unlike the case of MDP, the optimal policy $\pi_t$ for QMDP is a function of the history $h_t=(S_0,a_0,...,S_t)$ instead of simply the current state $S_t$ -- but all of the history $h_t$ is summarized in the quantile level $\tau_t.$ In other words, $\tau_t$ is a function (although not explicit) of the history and plays a role like that of a ``summary statistic.'' Theorem \ref{policy} tells us that the optimal policy $\pi_t$ is a function of only the current state $S_t$ and the ``summary statistic'' $\tau_t.$   Intuitively, this augmented quantile level $\tau_t$ reflects the historical performance of the MDP. A higher quantile level will encourage a more aggressive policy in the remaining periods while a lower quantile level will encourage conservative moves. For example, if we start with $\tau_0=\tau=0.5$, which means that our ultimate goal is to maximize the median cumulative reward over $0$ to $T$, then at some time $t$ in between, if we have already achieved a relatively high reward, i.e., a large $\sum_{k=0}^t r_k,$ the augmented quantile level $\tau_t$ will decrease to some value smaller than $0.5$ accordingly. This will drive us to take relatively conservative moves in the future, and vice versa.

\begin{theorem}[Optimal Policy]
We augment the state $S_t$ with a quantile $\tau_t$ to assist in the execution of the optimal policy. At the initial state $s_0$ and $\tau_0 = \tau,$ we define our initial policy function as
$$\pi_0: \mu_0(s_0,\tau_0) = \operatorname*{arg\,max}_a \  \widetilde{v}_0(s_0,\tau_0, a).$$

\noindent At time $t,$ we execute the output of $\mu_t$ and then the process reaches state $S_{t+1}$. Let $\mathbf{q}^*$ be the solution to the optimization problem  $OPT\left(S_t, \mu_t(S_t,\tau_t), \tau_t, v_{t+1}(\cdot,\cdot)\right)$. Here $v_{t+1}(\cdot,\cdot)$ is computed as in Theorem \ref{value}. The term $\tau_{t+1}$ is assigned as 
$$\tau_{t+1} = q^*_i$$
for the specific $i$ that satisfies $S_{t+1}=s_i.$ We define $\pi_{t+1}$ as 
$$\pi_{t+1}: \mu_{t+1}(S_{t+1},\tau_{t+1}) = \operatorname*{arg\,max}_a  \  \widetilde{v}_{t+1}(S_{t+1},\tau_{t+1}, a).$$
The policy $\pi=(\pi_0,...,\pi_T)$ defined above is the optimal policy for the objective (\ref{quantileMDP}) and obtains the optimal value $v_0(s_0,\tau_0)$.
\label{policy}
\end{theorem}

\subsection{QMDP and Other Risk Measures}

\label{illustration}

We compare QMDP to other risk measures using a simple example. Consider the following two-period gambling game: In time period one, a gambler participates in the game and receives or loses \$50 with equal probability; in time period two, the gambler has an option to participate in one of two fair games with, respectively, an equal chance of winning or losing \$20 or \$100. In the MDP framework, the two options are equivalent because both output a zero expected return. 


Using MDP with a nested risk measure \citep{ruszczynski2010risk, jiang2017risk}, the objective for this gambling game is
$$ \max_{a} \rho_\theta(r_1 + \rho_\theta(r_{2,a})),$$
where $r_1$ denotes the random reward in time period one, $a$ denotes the action, $r_{2,a}$ denotes the random reward in time period two, and $\rho_{\theta}(\cdot)$ is the risk measure to be specified by the decision maker where the parameter $\theta$ reflects the decision maker's risk attitude. Assuming the risk measure $\rho_{\theta}(\cdot)$ is monotonically non-decreasing, the optimal action $a$ is determined by
$$\max_{a} \rho_\theta(r_{2,a}).$$
The optimal action  for this model does not take into account the reward history. This allows for a dynamic programming algorithm that solves for the optimal decision (see \cite{ruszczynski2010risk} and \cite{jiang2017risk}), but it fails to capture the subsequent risk attitude of the gambler. If $r_1=\$50,$ the gambler might prefer to adopt the more conservative option in the second time period to guarantee winning at least \$30 by the end. On the other hand, if $r_1=-\$50,$ the gambler might prefer to participate in the risky game (\$100) to compensate for the loss; by taking the \$20 game in time period two, the gambler is doomed to lose money, whereas by taking the \$100 game, there is a chance of winning money at the end. To capture this type of risk attitude, the risk measure $\rho_{\theta}(\cdot)$ at time period two should be dependent on the outcome of $r_1$, which cannot be covered by a nested MDP formulation. 

\begin{figure}[!tbp]
  \centering
  \subfloat[Optimal value function]{\includegraphics[height=2 in]{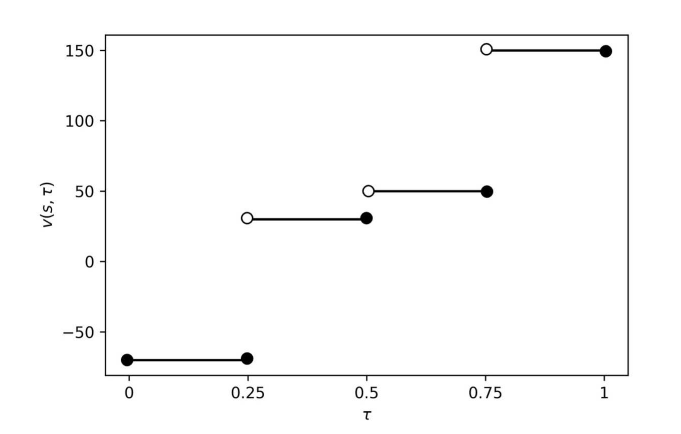}\label{fig:opt_val_gambling_game}}
  \hfill
  \subfloat[Optimal policy]{\includegraphics[height=2 in]{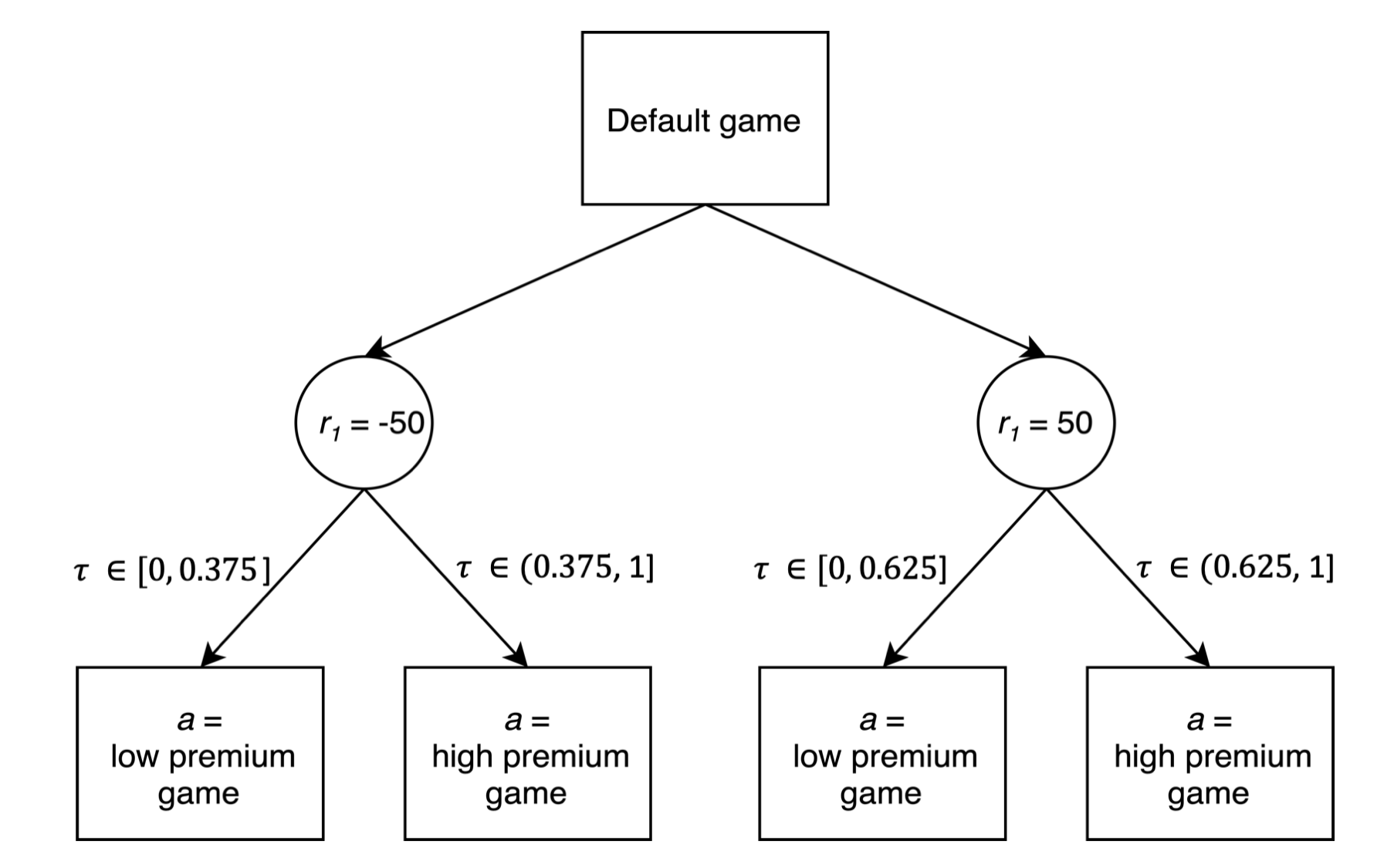}\label{fig:opt_policy_gambling_game}}
  \caption{QMDP optimal value function and optimal policy for two-period gambling game.}
  \label{fig:gambling_game_solution}
\end{figure}

In the QMDP model, the objective function is
$$\max_a Q_{\tau}\left(r_1+r_{2,a}\right)$$
where $\tau$ is the quantile level specified by the gambler. Figure \ref{fig:gambling_game_solution} shows the optimal QMDP value function, calculated as in Section \ref{solveQMDP}, and the optimal policy, obtained through forward execution following Theorem \ref{policy}. The optimal action in time period two is affected by both the risk parameter $\tau$ and the outcome of time period one. For example, if $\tau=0.4$, which is a slightly conservative attitude, then the gambler will participate in the less risky (\$20) game if $r_1=\$50$ but will participate in the risky (\$100) game if $r_1=-\$50$. If $\tau=0.6$, which is a slightly aggressive attitude, then the gambler will participate the less risky game if $r_1=\$50$ just as in the case of a conservative attitude; however, if there were no time period one but only time period two, then the gambler would participate in the riskier game for the 50\% chance of winning \$100. In QMDP, the risk attitude that governs the optimal action of each time period can change dynamically according to the outcomes in the past time periods. QMDP can search for the optimal policy in a more general class of (non-Markovian) policies than the nested risk measure models.

Using a utility function-based MDP formulation, the objective function is
$$\max_{a} u_\theta(r_1+r_{2,a}),$$
where $u_\theta(\cdot)$ is a utility function and subscript $\theta$ denotes risk attitude.
In such a model there is no transparent connection between the risk attitude (the choice of $\theta$) and the outcome $r_1+r_{2,a}$, and there is no clear characterization of the outcome or the objective function value under different choices of $\theta$ unless we repeatedly solve the problem with different specifications of $\theta$. The QMDP model provides a more explicit visualization of the return by characterizing the risk of the cumulative reward in a distributional manner. 

\section{Algorithms and Computational Aspects}

In this section, we present our algorithm for solving QMDP and discuss its computational aspects. As mentioned earlier, the key for computing the value function is to solve the optimization problem OPT. Thus, we first provide an efficient algorithm for solving OPT and then analyze its complexity.

\subsection{Algorithm for Solving the Optimization Problem OPT}

We formulate $OPT(s,\tau, a,  v_{t+1}(\cdot,\cdot))$ in a more general way as follows:

\begin{gather} 
 OPT \  \triangleq\  \max_{\mathbf{q}} \min_{i \in \{ q_i \neq 1 \mid i= 1,2,...,n\}}  \ \   g(i,q_i),
  \label{optGeneral} \\
  \text{subject to} \ \ \sum_{i=1}^n p_i q_i \le \tau,  \sum_{i=1}^n p_i=1,\nonumber\\
  q_i \in [0,1], \text{\ \ \ for \ \ } i =1,...,n.\nonumber 
\end{gather}
Here $\tau$ and the $p_i$'s are known parameters. The decision variable is $\mathbf{q}=(q_1,...,q_n).$ We introduce a function $g: \{1,...,n\} \times [0,1] \rightarrow \mathbb{R}$.  We assume that $g(i,\cdot)$ is a left continuous and piecewise constant function with finite breakpoints for all $i$. The variable $i$ refers to the state in the QMDP settings. We will show later that these assumptions are satisfied for value functions of QMDP with finite state space and discrete rewards. Therefore, we can represent and encode each function $g(i,\cdot)$ with a set of breakpoint-value pairs $$\left\{\left(b^{(1)}_i, v^{(1)}_i\right),... ,\left(b^{(n_i)}_{i}, v^{(n_i)}_{i}\right)\right\}$$ where $n_i$ is the number of pairs. Then we have
$$g(i, x) =   \left\{\begin{array}{ll}
              v^{(1)}_i, & \ \ \  \text{for}\ x \in \left[b_{i}^{(1)}, b_{i}^{(2)}\right], \vspace{0.1 cm} \\ 
              v^{(k)}_i, & \ \ \ \text{for}\ x \in \left(b_{i}^{(k)}, b_{i}^{(k+1)}\right] \text{ and } k = 2,...n_i.
                \end{array} \right. $$
Here we define $b_{i}^{(1)}=0$ and $b_{i}^{(n_i+1)}=1$ for all $i.$ 

\begin{algorithm}[ht!]
\caption{Algorithm for the Optimization Problem (\ref{optGeneral})}\label{euclid1}
\label{algoForOPT}
\begin{algorithmic}[1] 
\State \textbf{Input:} $\{p_i, g(i,\cdot) \mid i=1,..,n\}$ 
\State \textbf{Initialize}
$k_1=...=k_n=1,$  $\tau_{tmp}=0.$
\State Let $u_i = g(i,0)= v^{(1)}_{i}$, $i \in S = \{1,...,n\}$ and $u = \min_{i\in S} u_i.$ 
\State Let $f(0) = u.$
\While {$S \neq \emptyset$}
\State $S_0 = \text{argmin}_{i \in S} \ u_i$ 
\State $\tau_{new} = \tau_{tmp}$
\For {$i \in S_0$}
\If {$k_i=n_i$}
\State   $\tau_{new} = \tau_{new}  + p_i \left(1-b^{(k_i)}_{i}\right)$
\State $S = S \backslash \{i\}$
\Else
\State $\tau_{new} = \tau_{new}  + p_i \left(b^{(k_{i}+1)}_{i} - b^{(k_i)}_{i}\right)$
\State $k_i = k_i+1$
\State Update $u_i = v^{(k_i)}_{i}$
\EndIf
\EndFor
\If {$\tau_{tmp} = 0$}
\State Let $f(\tau) = u$ for $\tau \in [\tau_{tmp}, \tau_{new}]$
\Else
\State Let $f(\tau) = u$ for $\tau \in (\tau_{tmp}, \tau_{new}]$
\EndIf
\State Update $\tau_{tmp} = \tau_{new}$
\State Update $u = \min_{i \in S} \ u_i$
\EndWhile
\State \textbf{Return} $f(\cdot)$ 
\end{algorithmic}
\end{algorithm}

Algorithm \ref{algoForOPT} solves OPT. The idea is quite straightforward: we start with $q_i=0$ for all $i$ and gradually increase the $q_i$ that has the smallest value of $g(i,q_i)$ until the constraint $\sum_{i=1}^n p_i q_i \le \tau$ is violated. The $g(i,q_i)$ that has smallest value is the bottleneck for the objective function value. By increasing the corresponding $q_i,$ we keep improving the objective function value.  The output of the algorithm $f(\cdot)$ restores the optimal values of OPT as a function of $\tau \in [0,1]$.

We illustrate the algorithm with an example of $n=3$ in Figure \ref{fig:AlgorithmInterpretation}. In this example, we have three functions $g(i, \cdot)$ represented by three gray rectangles with corresponding transition probabilities denoted by $p_i$. We want to determine $f(\cdot)$, which is indicated by the red rectangle for each step. Following Algorithm \ref{algoForOPT}, we initialize input  $k_1 = k_2 = k_3 =1, \tau_{tmp} = 0$ and $u_1 = 10, u_2 = 8, u_3 = 10$. We then find (Step 1) $u = \min_{i\in S =\{1, 2, 3\}} u_i = 8$, and thus $f(0) = 8$. To find the upper bound $b^{(2)}$ for the value of 8, we execute the ``while" loop. The only $g(i, \cdot)$ that has value of 8 is $g(2,\cdot)$ so we assign $S_0 = {2}$ and $\tau_{new} = \tau_{tmp} = 0$. Since $k_2 =1$ and $n_2 = 2$ in this example, we can update $\tau_{new} = \tau_{new}  + p_i \left(b^{(k_{2}+1)}_{2} - b^{(k_2)}_{2}\right) = 0 + 0.5(0.4-0) = 0.2 $. Thus, in Step 1 we have $f(\tau) = 8$ for $\tau \in [0, 0.2]$. The algorithm keeps updating $f(\cdot)$ until the set $S$ becomes empty. In the end (rightmost panel of Figure \ref{fig:AlgorithmInterpretation}) we have fully specified $f(\cdot)$, and thus we have found the optimal values of OPT as a function of $\tau \in [0,1]$.

\begin{figure}[t]
\centering
\includegraphics[height=1.3 in]{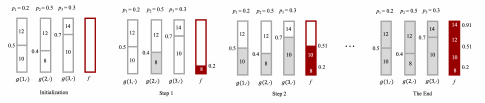}
\caption{\small Step-by-step execution of Algorithm \ref{algoForOPT} with $n=3$ sample $g(\cdot,\cdot)$ functions. Numbers inside and along the blocks represent the values and breakpoints of the input functions $g(i,\cdot)$. The shaded regions reflect the progress of the algorithm. In the end, the output is $f.$}
\centering
\label{fig:AlgorithmInterpretation}
\end{figure}

\subsection{Algorithm for Solving QMDP}

In this subsection, we summarize the previous results and provide the algorithm for solving QMDP as Algorithm \ref{algoForQMDP}. It is obtained by putting together Algorithm \ref{algoForOPT} with Theorems \ref{value} and \ref{policy}. One advantage of this dynamic programming algorithm is that the optimal value functions and the optimal policies at all states and quantiles are computed in a single pass. Indeed, this single-pass property is necessary for the quantile objective, because the optimal value and action at time $t$ could depend on the value function at time $t+1$ for all the quantiles. 

\begin{algorithm}[ht!]
\caption{Algorithm for Solving QMDP}\label{euclid}
\label{algoForQMDP}
\begin{algorithmic}[1] 
\State \textbf{Input:} Transition probabilities $\mathcal{P}(S_t, a, S_{t+1}),$ reward function $r_t(s_t, a_t, w_t),$ time horizon $T$.
\State \textbf{Computing Value Functions: } 
\State \textbf{Initialize}
Let $\mathcal{S}=\{s_1,...,s_n\},$ $v_{T}(s_i, \tau)=0$ for all $i = 1,...,n$ and $\tau \in [0,1].$
\For {$t = T-1,....,0$}
\For {$i=1,...,n$}
\For {$a \in \mathcal{A}_t$}
\State $p_{tmp}(s_j) = \mathcal{P}(S_t=s_i, a, S_{t+1}=s_j) $ for $j=1,...,n$
\State $v_{tmp}(s_j, \tau) = v_{t+1}(S_{t+1}=s_j, \tau) + r_t$ for $j=1,...,n$ and $\tau \in [0,1]$
\State $\widetilde{v}(s_i, \tau, a) = OPT\left(p_{tmp}(\cdot),v_{tmp}(\cdot,\cdot)\right)$
\EndFor
\State $v_t(s_i, \tau) = \max_a \widetilde{v}(s_i, \tau, a)  $
\EndFor
\EndFor
\State \textbf{Output:} $\{v_t(s, \tau), \widetilde{v}_t(s, \tau, a) \}_{t=0}^{T-1}$ for all $s \in \mathcal{S},$ $a \in \mathcal{A}$  and $\tau \in [0,1]$
\State \textbf{Execution: } 
\State \textbf{Initialize} $S_0 = s$ and our goal is to maximize $\tau$ quantile. Let $R=0$ and $\tau_0=\tau.$
\For {$t = 0,...,T-1$}
\State Take action $a_t = \arg\max_a  \widetilde{v}_t(S_t, \tau_t, a)$
\State Transit from $S_t$ to $S_{t+1} = s_j$ for some $j \in \{1,...,n\}$
\State $R = R+r_t(s_t,a_t,s_{t+1})$
\State Let $\mathbf{q}^*$ be the optimizer of $OPT(S_t, a_t, v_{t+1}(\cdot,\cdot))$
\State Update $\tau_{t+1}=\mathbf{q}^*_{j}$
\EndFor
\State \textbf{Output:} Cumulative reward $R$
\end{algorithmic}
\end{algorithm}

\subsection{Complexity Analysis and Approximation}

The computational cost of our algorithm for solving QMDP (Algorithm \ref{algoForQMDP}) is mostly concentrated in computing the value functions. It is easy to show that all the value functions are piecewise constant. This is because when the input functions of OPT are piecewise constant, the OPT procedure will output a piecewise constant function as well. Also, it can be readily seen that the complexity of Algorithm \ref{algoForOPT} is linear in the number of breakpoints for its output functions. Based on these facts, we have the following proposition.

\begin{proposition}
When the rewards are integer and bounded, $|r_t|\le R$ for all $t,$ then the complexity of Algorithm \ref{algoForOPT} for computing value functions for QMDP is $O\left(AST\cdot\max(RT,S)\right).$ Here $T$ is the length of the time horizon, and $A=|\mathcal{A}|$  and $S = |\mathcal{S}|$ are the sizes of the action and state space, respectively. 
\label{compl}
\end{proposition}

The proof of this proposition is straightforward: When the reward is integer and bounded by $R,$ the cumulative reward is bounded by $RT.$ Thus any value function has at most $RT$ breakpoints, which means that each call of OPT will induce at most $O(RT)$ complexity. Additionally, each call of OPT will have a  read and write complexity of $O(S)$. Therefore each iteration has $O(\max(RT,S))$ complexity. Since there are $AST$ iterations in total, the overall complexity is $O\left(AST\cdot\max(RT,S)\right).$ 

Though the algorithms work for both integer and non-integer cases, the analysis is more complicated when the rewards are non-integer because we have no simple way to bound the number of breakpoints for the value functions. The value function can become ``exponentially'' complicated as the backward dynamic programming proceeds, so that the cost to restore the value function will also grow exponentially. \cite{nemirovski2006convex} pointed out that the computation of the distribution of the sum of independent random variables is already NP-hard.
To prevent this explosion, one can either truncate the rewards to integers or create approximations of the value functions.  For the truncation approach, if we still want to preserve computational precision, we can scale up the rewards before truncation. For the approximation approach, we would restore the value function at $N$ uniform breakpoints. The choice of $N$ is up to the user and can be as large as, for example, $10,000,$  which means that we restore the value function only for all the quantile values with an interval of $0.0001$. 

From the above analysis, we observe that the bottleneck for the complexity of our algorithm lies in the complexity of the value function. In a traditional MDP, the value function is a function of the state $s$ and time stamp $t$. In QMDP, for each given $s$ and $t,$ we need to compute and retain the optimal values for all the quantiles in order to derive the value function for time $t-1$. Therefore, there is not much room for improvement on this complexity upper bound in a generic setting. In Appendix C, we present numerical experiments that further illustrate the computational aspect of our QMDP algorithm. QMDP is unsurprisingly more time-consuming compared to the conventional nested risk measure based model. However, as noted earlier, QMDP outputs the optimal value function and optimal policy in one single pass of dynamic programming, whereas an additional simulation procedure is needed to trade off risk versus reward for other risk-sensitive MDP models. This additional simulation may result in a computation time that is significantly larger than that needed to solve QMDP.

\section{Extensions}
In this section, we discuss several extensions of the QMDP model. We extend the model to solving CVaR MDP (Section \ref{CVar}) and present a time-consistency result for the quantile risk measure (Section \ref{riskmeasure}). We establish the optimal value and policy for the infinite-horizon case (Section \ref{InfMDP}).

\subsection{Conditional Value at Risk}
\label{CVar}
In this section, we show how the dynamic programming idea in QMDP extends to the CVaR objective. We follow the characterization of \cite{rockafellar2002conditional} for a definition of CVaR.
\begin{definition}
For $\tau \in (0,1),$ the conditional value at risk (CVaR) at level $\tau$ is defined as 
$$\text{CVaR}_{\tau}(X) \triangleq Q_{\tau}(X) + \frac{1}{1-\tau} \E\left[X-Q_{\tau}(X)\right]^+.$$
\label{cvarDef}
\end{definition}
We consider an alternative objective, that of maximizing the CVaR of the cumulative reward. 
\begin{equation}
\max_{\pi \in \Pi} \ \text{CVaR}^{\pi}_{\tau} \left[\sum_{t=0}^{T-1}r_t\left(S_t, a_t, w_t\right)\right].
\label{quantileMDP1}
\end{equation}
\cite{Bauerle2011} and \cite{Yu2017} solved the CVaR MDP problem via an augmented variable but their approach was computationally intensive. Our results using the QMDP model reveal the key step in the dynamic programming for CVaR MDP as an optimization problem that is similar to our OPT problem. \cite{chow2014algorithms} considered MDP with a CVaR constraint, \cite{Carpin} developed approximate algorithms for CVaR MDP under a total cost formulation, and \cite{chow2015risk} solved CVaR MDP for the case of an infinite horizon and discounted reward. The method presented here complements this line of literature, and the core part of our dynamic programming procedure shares the same spirit as the CVaR decomposition (proposed in \cite{pflug2016time} and later exploited by \cite{chow2014algorithms}).

As for the quantile objective, we define the value function
$$u_t^{\pi_{t:T}}(s, \tau)  \triangleq  \text{CVaR}^{\pi}_{\tau} \left[\sum_{k=t}^{T-1}r_k\left(S_k, a_k, w_k\right) \Big|S_t =s \right].$$
Here $\pi_{t:T} = (\mu_{t},...,\mu_{T-1})$ denotes the policy and the action $$a_k = \mu_k(\mathcal{H}'_k) = \mu_{k}(S_t,a_t,...,S_k)$$ for $k=t,...,T-1.$

\begin{theorem}[CVaR Value Function]
Let $\mathcal{S} = \{s_1,...,s_n\}$ and 
$u_{T}(s, \tau) = u'_T(s, \tau) = 0$ for all $s \in \mathcal{S} $ and $\tau \in (0,1).$
Then,
$$\tilde{u}'_{t}(s,\tau,a) = OPT(s, \tau, a, u_{t+1}'(\cdot,\cdot))$$
and $\mathbf{q}^*=(q_1^*,...,q_n^*)$ as the optimal solution to the OPT problem (dependent on $s$ and $\tau$).
  \begin{align*}
 \tilde{u}_{t}(s,\tau,a) & =   \frac{1}{1-\tau} \sum_{i=1}^n p_i(1-q_i)\left[ u_{t+1}(s_i,q_i^*) + r_t\left(s, a, w_t\right)\right] \\ 
  & =  \frac{1}{1-\tau} \sum_{i=1}^n p_i(1-q_i)\left[ u_{t+1}(s_i,q_i^*) +  r_t\left(s, a, l_t(s, a, s_i)\right)\right]
\end{align*}
Here $w_t=l_t(s, a, s_{t+1})=l_t(s, a, s_i)$ is from Assumption \ref{assume} (c). By taking maximum over the action $a$,
$${u}_t(s,\tau) = \max_{a}  \widetilde{u}_t(s,\tau,a).$$
Denote the optimal action as $a^*$ (dependent on $s$ and $\tau$). 
$${u}'_t(s,\tau) = \widetilde{u}_t'(s,\tau,a^*).$$
In this way, we have 
$$u_0(s,\tau) = \max_{\pi \in \Pi} \ \text{CVaR}^{\pi}_{\tau} \left[\sum_{t=0}^{T-1}r_t\left(S_t, a_t, w_t\right)\right].$$
\label{valFunc1}
\end{theorem}

Theorem \ref{valFunc1} presents a dynamic programming formulation for the CVaR MDP problem. The key observation is that the CVaR definition involves the quantile, and the results developed Section 4 provide useful tools for quantile-related computations. The functions $u_t$ and $u_t'$ in Theorem \ref{valFunc1} represent the optimal CVaR cost-to-go value function and the corresponding quantiles of the cumulative reward. In contrast to QMDP, the formulation here takes the maximum of the CVaR function and updates the corresponding quantile function according to the optimal action. This result provides a finite-horizon solution that complements the infinite-horizon solution in \cite{chow2015risk}.

\subsection{Dynamic Risk Measures}
\label{riskmeasure}
The quantile objective, as a dynamic risk measure, has been criticized for its time-inconsistency \citep{cheridito2009time, iancu2015tight}. In fact, the quantile objective specifies a family of risk measures (functions) parameterized by the quantile level $\tau$ and thus the conventional notion of time-consistency no longer fits. In Theorem \ref{time}, we present a time-consistency result for the quantile risk measure. The result is implied by the dynamic programming results developed in the previous sections. 

\begin{theorem}
Given two real-value Markov chains with a finite state space, $\{X_{t}\}_{t=0}^{T-1}$ and $\{Y_{t}\}_{t=0}^{T-1}$ and a function $r: \mathbb{R} \rightarrow \mathbb{R},$ if
\begin{equation}
    Q_{\tau}\left(\sum_{t=0}^{T-1}r(X_t)\Big |\mathcal{F}_{k}\right) \ge Q_{\tau}\left(\sum_{t=0}^{T-1}r(Y_t)\Big |\mathcal{F}_{k}\right) 
    \label{timeIneq}
\end{equation}
holds for $k=1,...,T-1$ and all $\tau \in (0,1)$, where $\mathcal{F}_{k}$ denotes the $\sigma$-algebra generated by $\{(X_t,Y_t)\}_{t=0}^k$, and if $X_0$ and $Y_0$ have an identical distribution, then we have
$$Q_{\tau}\left(\sum_{t=0}^{T-1}r(X_t)\right) \ge Q_{\tau}\left(\sum_{t=0}^{T-1}r(Y_t)\right).$$
\label{time}
\end{theorem}

$\{X_{t}\}_{t=0}^T$ and $\{Y_{t}\}_{t=0}^T$ can be interpreted as two Markov chains specified by an MDP with two fixed policies. Condition (\ref{timeIneq}) in the above theorem can be viewed as parallel to the dynamic risk measure in \cite{cheridito2009time} and \cite{iancu2015tight}, but it is a stronger condition because here the inequality is required to hold for all $\tau \in [0,1]$. We can see that this is entailed in the dynamic programming procedure for solving QMDP, where the quantile value at time $t$ depends on the quantile values at time $t+1$ for all $\tau\in (0,1)$ in general. This explains why the quantile objective, as a risk measure, is not time-consistent in the conventional sense, where we only require that inequality (\ref{timeIneq}) holds for a fixed $\tau$. It emphasizes that the optimization of a dynamic quantile risk measure requires changing the risk measure parameter (the quantile level $\tau_t$ in Theorem \ref{policy}) itself over time. With the stronger condition (\ref{timeIneq}), the quantile objective can also be viewed as a time-consistent risk measure. 

The changing of the risk measure parameter can be seen in the example in Section \ref{illustration}. We know that the risk parameter $\tau_t$ in the QMDP model changes over time according to both $\tau_{t-1}$ and the reward outcome $r_{t-1}$. In the example, if $\tau_0=\tau=0.4$ and $r_1=50,$ then $\tau_1=0.3$. A risk-averse decision maker will become more risk-averse if a good return is achieved in the first time period under the quantile objective. The dynamic changing of risk parameter cannot be captured by conventional time-consistent risk measures such as the nested risk measure that require the same risk parameter over the entire time horizon. In this way, the quantile objective enriches the family of time-consistent risk measures.

\subsection{Infinite-Horizon QMDP}
\label{InfMDP}

For the infinite-horizon QMDP, the objective is 
\begin{equation*}
\max_{\pi \in \Pi} Q^{\pi}_{\tau} \left[\sum_{t=0}^{\infty} \gamma^t r_t\left(S_t, a_t, w_t\right)\right].
\end{equation*}
Here $r_t = r(S_t,a_t, w_t)$ is stationary and $\gamma \in (0,1)$ is the discount factor. The policy $\pi = \{\mu_t\}_{t=0}^\infty$ consists of a sequence of decision functions and $\mu_t$ maps the historical information $h_t = (S_0,a_0,...,S_{t-1},a_{t-1},S_t)$ to an admissible action $a_t \in \mathcal{A}_t \subset \mathcal{A}.$ The value function is
\begin{equation}
v(s,\tau) \triangleq \max_{\pi \in \Pi} Q^{\pi}_{\tau} \left[\sum_{t=0}^{\infty} \gamma^t r_t\left(S_t, a_t, w_t\right) \Big | S_0 = s\right].
\label{valFuncInf}
\end{equation}

Similar to the infinite-horizon MDP, we propose a value iteration procedure to compute the QMDP value function.  The result is formally stated as Theorem \ref{valueInf}. We use $k$ to denote the iteration number here to distinguish it from the index notation $t$ in Theorem \ref{value} which is the time stamp for backward dynamic programming.
\begin{theorem}[Infinite-Horizon Optimal Value Function] Consider the following value iteration procedure:
$$v^{(0)}(s,\tau) = 0,$$
$$\widetilde{v}^{(k+1)}(s,\tau,a)  =  OPT(s,  \tau, a, \gamma  v^{(k)}(\cdot,\cdot) ),$$
  $$ v^{(k+1)}(s,\tau) = \ \max_a \  \widetilde{v}^{(k+1)}(s,\tau,a).$$
Then we have 
$$\lim_{k\rightarrow \infty} v^{(k)}(s,\tau) = v(s,\tau),$$
for any $s\in \mathcal{S}$ and $\tau\in [0,1].$ Furthermore, since the function $v(s,\tau)$ is a monotonic function for $\tau$, the convergence is uniform.
\label{valueInf}
\end{theorem}

\noindent The key to the proof of the theorem is to show that the OPT procedure, as an operator, features the same contractive mapping property as the Bellman operator in a traditional MDP. The contraction rate is simply the discount factor $\gamma.$ Based on the optimal value function, we have the following result characterizing the optimal policy. 

\begin{theorem}[Infinite-Horizon Optimal Policy] Let $v(\cdot, \cdot)$ be the optimal value function as in Theorem \ref{valueInf} and
$$\widetilde{v}(s,\tau,a) \triangleq OPT(s,  \tau, a, \gamma  v(\cdot,\cdot)).$$
We augment the state $S_t$ with a quantile $\tau_t$ to assist the execution of the optimal policy. At the initial state $s_0$ and $\tau_0 = \tau,$ we define our initial policy function as
$$\mu_0(s_0,\tau_0) = \operatorname*{arg\,max}_a \  \widetilde{v}(s_0,\tau_0, a).$$
At time stamp $t,$ we execute $\pi_t$ and then arrive at state $S_{t+1}$. Let $\mathbf{q}^*$ be the solution to the optimization problem  $OPT\left(S_t, \mu_t(S_t,\tau_t), \tau_t, \gamma v(\cdot,\cdot)\right)$. The term $\tau_{t+1}$ is defined as 
$$\tau_{t+1} = q^*_i,$$
for the specific $i$ such that $S_{t+1}=s_i,$ and $\mu_{t+1}$ is defined as 
$$\mu_{t+1}(S_{t+1},\tau_{t+1}) = \operatorname*{arg\,max}_a  \  \widetilde{v}(S_{t+1},\tau_{t+1}, a).$$
The policy $\pi=\{\mu_{t}\}_{t=0}^\infty$ is the optimal policy for the objective (\ref{valFuncInf}) and obtains the optimal value $v(s_0,\tau_0)$.
\label{policyInf}
\end{theorem}

The value iteration procedure is similar to the backward dynamic programming procedure for the finite-horizon case. This is because we can always interpret the finite-horizon reward as an approximation of the infinite-horizon reward by truncating the reward after time $T$. It is worth noting that CVaR does not break the value iteration aspect of the infinite-horizon case. Therefore the results in the previous subsection for the infinite-horizon counterpart of CVaR MDP also hold, and the value iteration procedure provides an alternative solution to the infinite-horizon CVaR MDP problem \citep{chow2015risk}.

\section{Empirical Results}
We present two sets of empirical results, evaluating our model on both a synthetic example and a problem of HIV treatment initiation.

\subsection{Synthetic Experiment}

\subsubsection{Overview}
We construct a synthetic example and perform simulations to illustrate the computational complexity of QMDP as a function of the state size, time horizon, and reward structure with comparison to MDP and two risk-sensitive MDP models. We also demonstrate how the QMDP model can be applied for risk assessment of an MDP.

\subsubsection{Model Formulation}
In this example, a player moves along a chain and receives rewards dependent on his location. Figure \ref{fig:simpleMDP} illustrates the model for this game. The arrows represent the possible movements. At each time step (s)he takes the action to stay or to move. If the player chooses to move, (s)he will move randomly to a neighboring state. The goal is to maximize expected cumulative reward over the time horizon. 
 
\begin{figure}[ht!]
\centering
\includegraphics[height=1in]{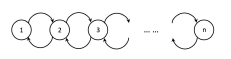}
\caption{Illustration of the simple QMDP model.}
\centering
\label{fig:simpleMDP}
\end{figure}

We formulate the model in the language of MDP as follows:
\begin{itemize}
\item \textbf{\textit{Time Horizon}}: We assume there are $T$ decision periods.

\item \textbf{\textit{State}}: We denote the state by $S_t,$ $t=0,..,T$ where  $S_t \in \mathcal{S} = \{1,...,n\}$.

\item \textbf{\textit{Action}}: At each time $t$, the player takes an action $a_t\in \mathcal{A} = \{\text{Stay, Move}\}.$

\item \textbf{\textit{Transition Probability}}: 
\begin{itemize}
\item When $a_t=\text{Stay},$ the player will stay at his location with probability $1$.
\item When $a_t = \text{Move},$ the player will move  randomly to one of its neighbors. When the player starts from the ends of the chain, then (s)he moves to her/his single neighbor with probability 1.
\end{itemize}
\item \textbf{\textit{Rewards}}: When the player stays in state $i$ at the beginning of a time period, (s)he receives a reward $R_i.$
\end{itemize}

\subsubsection{Results}
We ran $10^5$ simulation trials solving QMDP and MDP (on a laptop with a 2.8 GHz Intel Core i7). In each simulation trial, the transition probabilities were randomly generated and the rewards were randomly sampled integers no greater than $R_{max}.$ Figure \ref{fig:TimeComplexity} shows the average computation time as a function of the time horizon $T$, the number of states $n$, and the maximum reward amount $R_{max}$. The CPU time for solving QMDP is quadratic in the time horizon $T$ and linear in the state size $n$, and grows linearly with maximum reward amount and then fluctuates after $R_{max}$ reaches a certain level. This does not contradict the complexity analysis in Proposition \ref{compl}; the quadratic complexity in $R_{max}$ is an upper bound but is not necessarily tight for every trial. Figure \ref{fig:TimeComplexity} shows that, as expected, MDP is more time efficient than QMDP since QMDP records the full distribution of cumulative reward at each step of dynamic programming whereas MDP only records the mean value.

\begin{figure}[ht!]
\centering
\includegraphics[height=2.5 in]{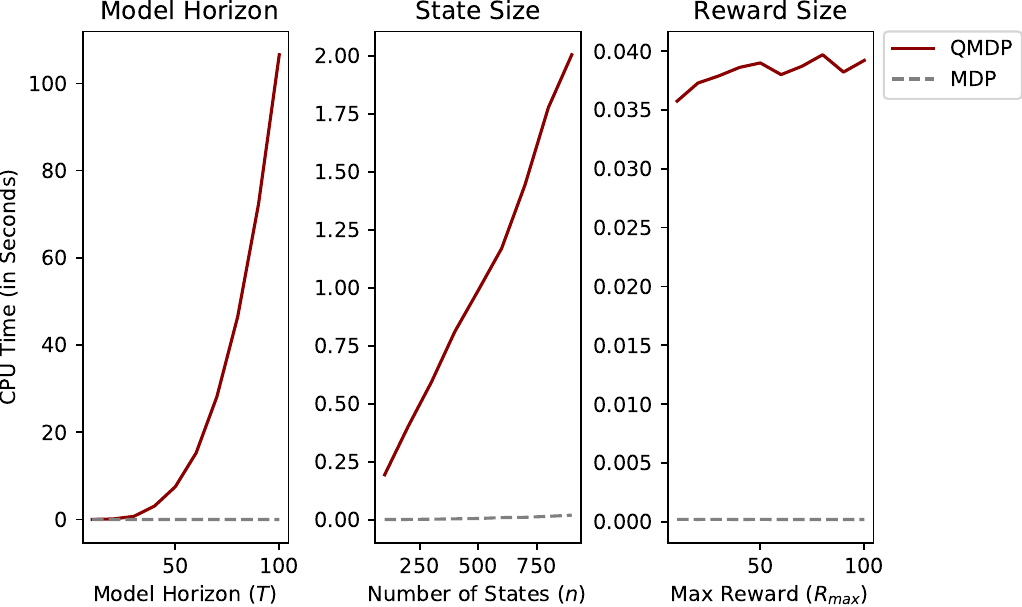}
 \caption{\small{Synthetic example: CPU time of QMDP and MDP. Base model parameters: time horizon $T$ = 10, state space size $n$ = 20, max reward $R_{max}$ = 10. For each experiment, we changed a single parameter and monitored the running time. The dark red solid lines indicate the CPU time for execution of the QMDP algorithm and dashed gray lines indicate the CPU time for execution of the MDP algorithm.}}
\label{fig:TimeComplexity}
\end{figure}

\begin{figure}[ht!]
\centering
\includegraphics[height=2.5 in, width = 6 in]{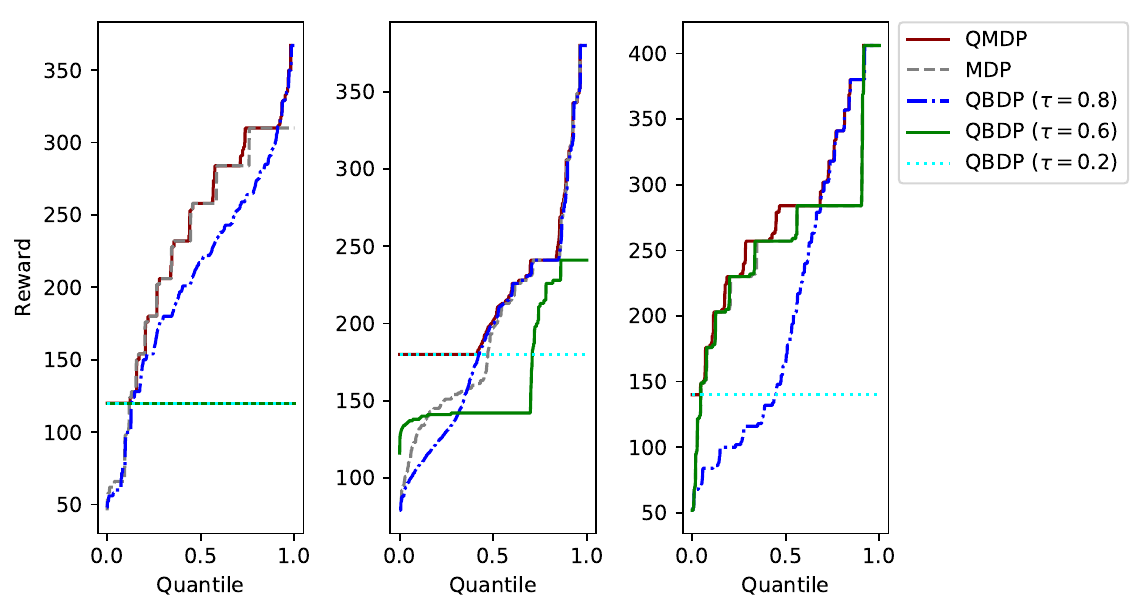}
\caption{\small{Synthetic example: QMDP value function comparison with MDP and QBDP. Each plot is obtained from a different initialization of model parameters. The gray dashed lines are the cumulative density function for simulations with the execution of the optimal MDP policy. The red lines are the optimal quantile rewards computed via QMDP. The remaining lines are the cumulative density functions obtained by simulating the optimal policies from QBDP with different preset values of $\tau$.}}
\label{fig:QBDPReward}
\end{figure}

\begin{figure}[ht!]
\centering
\includegraphics[height=2.5 in, width = 6 in]{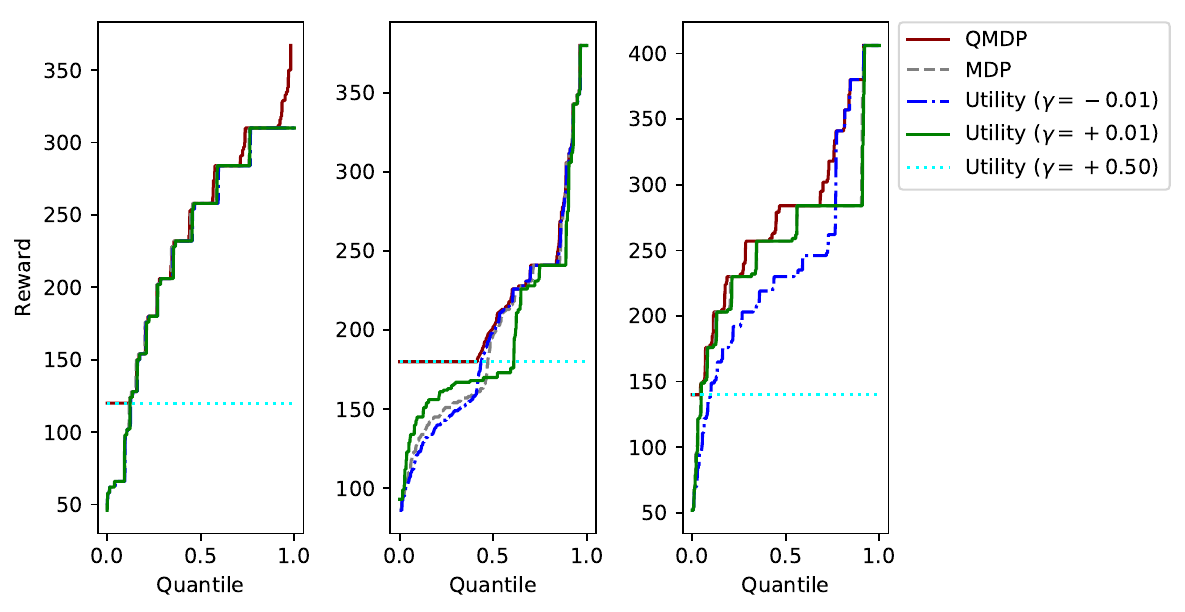}
\caption{\small{Synthetic example: QMDP value function comparison with MDP and utility function-based MDP. Each plot is obtained from a different initialization of the model parameters. The gray dashed lines are the cumulative density function for simulations with the execution of the optimal MDP policy. The red lines are the optimal quantile rewards computed via QMDP. The remaining lines are the cumulative density functions obtained by simulating the optimal policies from the utility function-based approach with different preset values of $\gamma$.}}
\label{fig:UtilityReward}
\end{figure}

In addition, we implemented two other risk-sensitive MDP models. We applied the nested composition of a one-step risk measure \citep{jiang2017risk, ruszczynski2010risk}, which aims to solve the following objective with $\rho_{\tau}$ as a quantile operator for corresponding value at specified $\tau$:
\begin{equation}
\max_{\pi \in \Pi} \rho_{\tau} [r_1 +\rho_{\tau} (r_2 + \dots +\rho_{\tau} (r_T))].
\label{quantileMDP2}
\end{equation}
We will refer to this as quantile-based dynamic programming (QBDP). We also implemented a utility function-based approach \citep{howard1972risk, chow2017} which solves the MDP with exponential utility function  $u(v) = - \frac{v}{|v|}e^{-\gamma v} $. The parameter $\gamma$ indicates the risk attitude of the decision maker: a negative value of $\gamma$ indicates that the decision maker is risk-seeking, whereas a positive $\gamma$ means the decision maker is risk-averse.

Figures \ref{fig:QBDPReward} and \ref{fig:UtilityReward} compare the outcome of QMDP with MDP and with QBDP and the utility function-based MDP, respectively. We generated three random simulation trials corresponding to the three panels in each figure. To obtain the CDF for the cumulative reward under $\pi^*$, we simulated $20,000$ instances and plotted the empirical histogram as the gray dashed line. We plotted the best quantile reward obtained from QMDP as the red line. A point $(q,r_0)$ on the gray dashed line means that the policy $\pi^*$ can achieve at least $r_0$ cumulative reward with probability $1-q.$ A point $(q,r_1)$ on the red line means that the optimal $q$-quantile reward is $r_1,$ i.e., there exists a policy that can achieve at least $r_1$ cumulative reward with probability $1-q.$ For the QBDP and utility function-based approaches, we solved the problem by setting various values for the preset parameters ($\tau$ in QBDP and $\gamma$ in the utility function approach) and then simulating $20,000$ instances to obtain the CDF of cumulative reward associated with the corresponding policies.

The QMDP value function tells the optimally achievable quantile values for all quantiles. In contrast, the optimal MDP value function only concerns the expectation, and the optimal QBDP value function has no explicit connection to the cumulative reward. To interpret the corresponding policy in a traditional MDP or QBDP, we need to run simulations and plot the histogram of the cumulative reward. The computational cost of this simulation procedure may offset the computational advantage of such models. In Appendix C we compare the computation time of QMDP versus QBDP for the synthetic example. Moreover, most risk-sensitive MDP models, like QBDP, only provide a glimpse of the inherent risk by solving the MDP problem with a single risk parameter. A procedure to trade off the risk and reward is then needed to select a proper risk parameter. QMDP simplifies the procedure by providing the information for all quantiles at once.


\subsubsection{MDP Risk Assessment}
In Figures \ref{fig:QBDPReward} and \ref{fig:UtilityReward}, we observe that the red curve, by definition of the QMDP, is never below the gray curve for any quantile. The gap between the curves indicates the space for improvement of QMDP over MDP for any given quantile, and thus helps us understand the inherent risk associated with the optimal MDP policy. Specifically, 

\begin{itemize}
    \item For the example on the left in both figures, the only significant difference occurs at the lowest quantiles. In this case the policy $\pi^*$, although not necessarily achieving quantile optimality, is quite stable and robust.
    \item For the example in the middle in both figures, if a quantile reward is desired, there is an opportunity for significant improvement for quantiles below the median. In this case, a risk-sensitive MDP model might be desirable for a risk-averse decision maker.
    \item For the example on the right in both figures, small differences occur throughout, indicating that if a quantile reward is desired, QMDP can achieve a somewhat better solution. In general, if the gap between the two curves is not significant, the traditional MDP should be used since it guarantees the optimal expected return in addition to achieving a near-optimal quantile reward. 
\end{itemize}   



\begin{figure}[t]
\centering
\includegraphics[height=1.85 in]{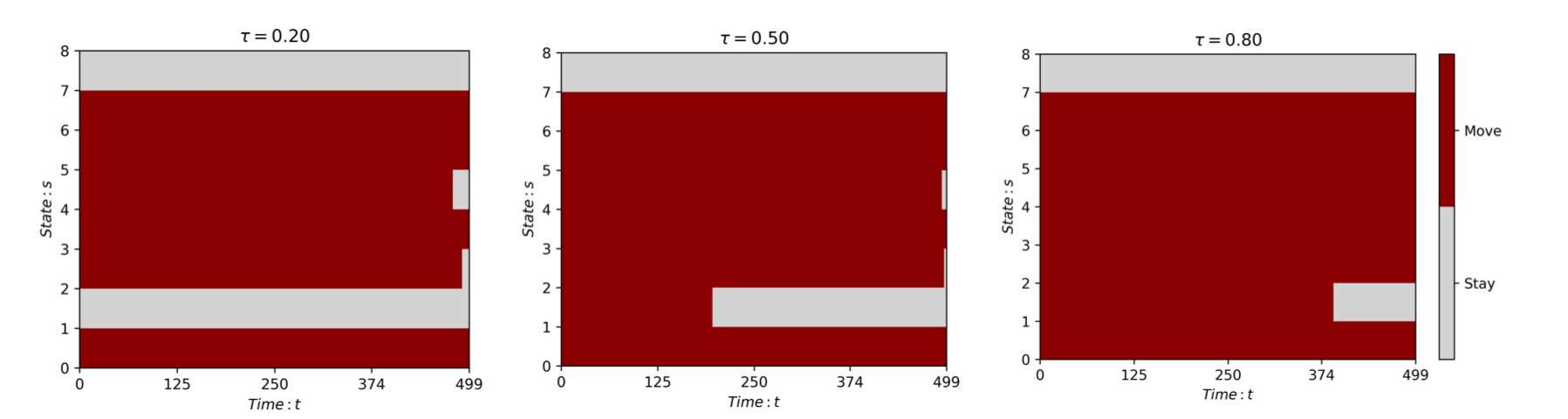}
\caption{Synthetic example: Optimal QMDP actions at different states ($s$) and different time periods ($t$) with different $\tau$ values.}
\centering
\label{fig:simple_MDP_policy}
\end{figure}

QMDP also provides risk interpretation of each state. Figure \ref{fig:simple_MDP_policy} shows the optimal QMDP actions for an instance of the game shown in Figure \ref{fig:simpleMDP}. In this problem instance, $T=500$, $|\mathcal{S}|=8$, and the reward vector is $(1, 10, 2, 0, 7, 9, 12, 18).$ We considered $\tau=0.2,$ $0.5,$ and $0.8$. The color of the point at location $(t,s)$ denotes the optimal action when the state at time $t$ is $s$. Note that the optimal actions under all quantiles are obtained as the output of QMDP in one shot, as a byproduct of the optimal value function. The highest reward, $18$, is obtained in state $8$ so, intuitively, the optimal decision is to move until $s=8$. However, state $2$ also provides a good reward of $10,$ so the player may want to stay in state $2$ if there are not many remaining time periods because it takes certain amount of time to traverse from state $2$ to state $8$ and less reward can be collected during the process. The decision of whether to move from state $2$ also depends on the player's risk attitude: a risk-averse player ($\tau = 0.2$) would choose to stay while a risk-seeking player would choose to move. The optimal action plot in Figure \ref{fig:simple_MDP_policy} provides an understanding of this state-level inherent risk and provides guidance for balancing the risk and reward for different risk attitudes and time horizons.

\subsection{Case Study: HIV Treatment Initiation }
\subsubsection{Background}
An estimated 37 million people worldwide are living with HIV \citep{WHO_HIV}.  Effective antiretroviral therapy (ART) reduces HIV-associated morbidity and  mortality for treated individuals \citep{tanser2013high} and has transformed HIV into a chronic disease. However, there is debate around the optimal time to initiate ART because of potential long-term side effects such as increased cardiac risk \citep{Freiberge003411}. Patients who delay initiating ART may sacrifice immediate immunological benefits but avoid future side effects. 

\cite{Diana2011} constructed a  sequential decision model to determine the ART initiation time for individual patients that maximizes expected quality-adjusted life expectancy, taking into account the potential for long-term side effects of ART (increased cardiac risk). However, the MDP model used in the analysis cannot capture patients' risk attitudes, which may affect their preferences regarding treatment \citep{Fraenkel}. QMDP bridges the gap between the traditional MDP and the patient's risk attitude by allowing for different values of the quantile threshold in the QMDP objective to incorporate the risk preferences. 

\cite{Diana2011} and many other MDP healthcare applications have considered the impact of parameter uncertainty on the optimal policy. Some of these efforts can be characterized as robust MDP frameworks \citep{Nilim2005,doi:10.1287/moor.1120.0566, Zhang2015RobustMD}. QMDP differs in that it provides a unique perspective on the inherent risk of the original MDP. The analysis explicitly reveals the uncertainty of the cumulative reward and allows the analyst to determine whether a risk-sensitive MDP framework should be used.

\subsubsection{Model Formulation}
The QMDP formulation of the optimal ART initiation time problem is straightforward and is similar to the MDP formulation. 

\begin{itemize}
\item \textbf{\textit{Time Horizon}}: We assume the patient is assessed at each time period $t \in \{0, 1, 2, \cdots, T\}$.\
\item \textbf{\textit{State}}: We characterize the state of the patient at time $t$ as $S_t = (c_t, y_t, d_t)$. The state is a function of the patient's CD4 cell count $(c_t)$ (which is a measure of the current strength of the patient's immune system), age $(y_t)$, and ART treatment duration $(d_t)$. In addition, we create an absorbing state for death, $D$. We divide the continuous CD4 cell counts into $L$ bins, $C = \{C_1, C_2, \cdots, C_L\}$. For age, we have $y_t \in [Y_0, Y_N]$, where $Y_0$ is the starting age and $Y_N$ is the terminal age of the patient. For treatment duration, $d_t = 0$ indicates that the patient has not yet started ART. Once the patient has started ART, $d_t$ increases by one unit after each time step. \

\item \textbf{\textit{Action}}: At each time $t$, the patient takes an action $a_t\in \{W,Rx\}$, where $W$ represents waiting for another period and $Rx$ means starting ART treatment immediately (and remaining on ART for life). \

\item \textbf{\textit{Transition Probability}}: The transition probability $P_k(S_{t}, a_t, S_{t+1})$ depends on the patient's current state $S_{t}$, the action $a_t$ at time $t$, and the state $S_{t+1}$. Two types of transitions can occur: transition between different CD4 count levels and transition to the terminal (death) state $D$.\

\item \textbf{\textit{Rewards}}: Two types of rewards are accrued: an immediate reward and a terminal reward. The immediate reward ($R^I$) is measured as the quality-adjusted life years (QALYs) the patient experiences when transitioning from $S_t$ to $S_{t+1}$ ($S_t \in C, S_{t+1} \in \{C,D\}$). We assume that if death occurs (that is, the patient transitions to state $D$ in period $t+1$) its timing is uniformly distributed from $t$ to $t+1$; in this case, we halve the immediate reward associated with state $S_t$. The terminal reward ($R^E$) is the cumulative remaining lifetime QALYs for a patient who passes the terminal age ($Y_N$).
\end{itemize}

We instantiated the model for the case of HIV-infected women in the United States, aged 20 to 90 years old. We grouped CD4 count levels into 7 bins: 0-50, 50-100, 100-200, 200-300, 300-400, 400-500, $>$500 cells/mm$^3$. Lower CD4 counts indicate greater disease progression, with CD4 counts at the lowest levels typically corresponding to full-blown AIDS. Each time period represents half a year. Every six months a patient can choose to start ART immediately or delay for another six months. To obtain the cumulative QALYs after the terminal age ($R^E$), we used a validated HIV natural history model \citep{zhong2020} and performed a cohort simulation that utilized the same model parameters including transition probabilities and utilities (quality-of-life multipliers). Values for all model parameters are provided in Appendix D. 

\subsubsection{Results}
 
We considered QMDP models with three different quantile thresholds, $\tau = 0.2$, $0.5$, and $0.8$. As $\tau$ increases, the patient becomes less risk-averse. Figure \ref{fig:HIV_optimal_policy} shows the optimal actions as a function of age and CD4 count.  Similar to the findings from the MDP model \citep{Diana2011}, we find that patients who are older or who have high CD4 counts tend to delay ART. In both cases, the reduction in HIV-associated morbidity and mortality from initiating ART is outweighed by the induced cardiac risks. For older patients, the induced cardiac risks are substantial because of the higher baseline cardiac risks at older ages. Patients with high CD4 counts are relatively healthy so the benefits from starting ART are less than the induced cardiac risks. 

In contrast to an MDP analysis, which maximizes expected cumulative reward and does not consider risk preferences, Figure \ref{fig:HIV_optimal_policy} shows that different risk attitudes of patients will lead to different treatment preferences. Patients who are less risk-averse (i.e., patients with higher levels of $\tau$) will tend to start ART sooner than patients who are more risk-averse. For example, a risk-averse 60-year-old woman with a CD4 count of 200 will choose to delay ART initiation, whereas a less risk-averse woman with the same CD4 count would choose to start ART. This is because patients with lower levels of risk aversion are more willing to accept elevated cardiac risks in order to gain the immunological benefits of ART. By incorporating the patient's risk attitude,  QMDP allows for a patient-centered care plan.


We note that instabilities still exist in the computed QMDP policies for this example, especially within regions where an action switch is made. This phenomenon is similar to a non-monotonic policy achieved in an MDP or robust MDP where some of the sufficient conditions for a monotonic policy are violated \citep{Zhang2015RobustMD}. The instability may be caused by the structure of the rewards (immediate and terminal) and/or the transition probabilities of the underlying simulation model. Further research is needed to determine sufficient conditions for a monotonic optimal policy for QMDP. 

\begin{figure}[ht!]
\centering
\includegraphics[height=1.8in]{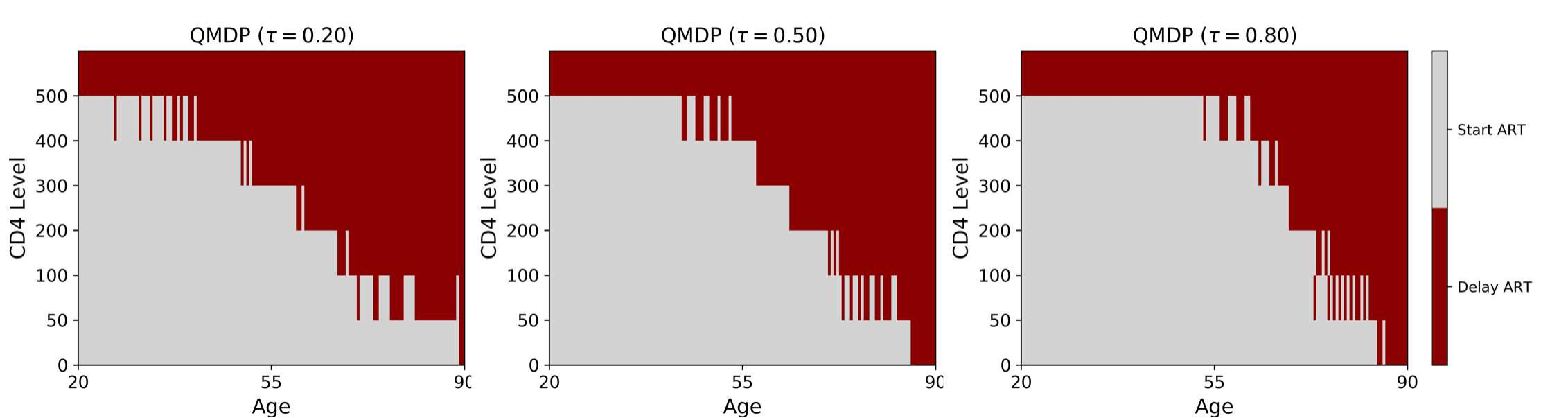}
\caption{HIV treatment example: Optimal actions for QMDP with $\tau = 0.20, 0.50$, and $0.80$.}
\centering
\label{fig:HIV_optimal_policy}
\end{figure}

\begin{figure}[t]
\centering
\includegraphics[height=1.35in]{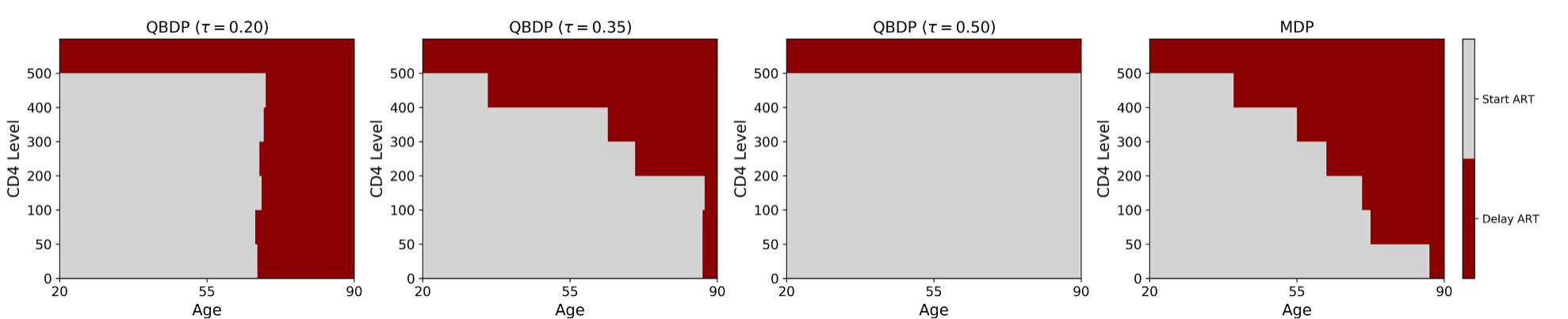}
\caption{HIV treatment example: Optimal actions for QBDP (for $\tau = 0.20,$ $0.35,$ and $0.50$) and MDP.}
\centering
\label{fig:HIV_QBDP_MDP_policy}
\end{figure}

Figure \ref{fig:HIV_QBDP_MDP_policy} shows the optimal actions for the HIV treatment example obtained using MDP, and QBDP with three values of $\tau$. The color of each point indicates the optimal action, to initiate ART or delay, for different CD4 levels and ages. The optimal MDP action roughly matches the optimal QMDP action with $\tau=0.5$ (see Figure \ref{fig:HIV_optimal_policy}). This follows the intuition that the median roughly equals the mean of a symmetric distribution. However, for QBDP, the optimal action roughly matches the optimal MDP action when $\tau = 0.35$. The risk parameters in risk-sensitive MDP models such as QBDP are less  interpretable than the quantile parameter in QMDP: although the parameter $\tau$ conveys a meaning of the quantile of instant reward, the QBDP optimal value function has nothing to do with the $\tau$-quantile or other distributional information of the true cumulative reward under the QBDP optimal policy.


In Section \ref{illustration}, we noted that the optimal policy of QBDP or other nested-risk-measure models does not consider past history and thus cannot consider non-Markovian policies. Our experimental results reveal another potential shortcoming of the QBDP model. When computing the optimal value function and optimal policy, the QBDP model summarizes the cost-to-go cumulative reward into a deterministic value using backward dynamic programming. An improper choice of the risk parameter in QBDP may cause a loss of information for the future cost-to-go reward and thus cannot distinguish the outcomes under different actions. For example, different actions may result in the same median but different 20\% quantile values. In this case, if we choose $\tau=0.5$ for QBDP, the optimal policy is to always initiate treatment when the patient's CD4 level is less than 500, regardless of age. QBDP does not take into account the information from lower quantiles and cannot distinguish different actions, leading to overly myopic decisions. In contrast, QMDP summarizes the future cost-to-go value as a function of the quantile, and memorizes all the quantiles of the cost-to-go value, and thus is more informative.


\section{Discussion}

We have presented a novel quantile framework for Markov decision processes in which the objective is to maximize the quantiles of the cumulative reward. We established several theoretical results regarding quantiles of random variables which contribute to an efficient algorithm for solving QMDP. The examples we presented show how solving QMDP with different values of $\tau$ generates solutions consistent with different levels of risk aversion. 

The QMDP model can be applied to a variety of problems in areas such as health care, finance, and service management where decision robustness and risk awareness play a key role. In this paper, we have restricted our attention to obtaining an exact solution for the QMDP problem. The complexity of our algorithm is $O\left(AST\cdot\max(RT,S)\right)$, where $T$ is the length of the time horizon, and $A=|\mathcal{A}|$ and $S = |\mathcal{S}|$ are the sizes of the action and state space, respectively. A promising area for future research is to determine how the QMDP model can be applied to very large scale real-world problems. In other risk-sensitive MDP settings, approximate dynamic programming (ADP) methods (e.g., \cite{jiang2017risk}) have been used to address the issue of exploding state space and inefficient sampling in large-scale problems. Further research could investigate how to incorporate ADP methods into the QMDP model. 

In addition to its value in determining the optimal decisions associated with different levels of risk aversion, QMDP provides a useful adjunct to MDP. Comparing the QMDP solution for different values of $\tau$ to the CDF of the MDP reward reveals the improvement that could be gained over the MDP solution if a quantile criterion were used. Depending on the decision maker's risk attitude, one might want to instead use QMDP or another risk-sensitive MDP model.



\section*{Acknowledgement}
This work was partially supported by Grant Number R37-DA15612 from the National Institute on Drug Abuse. 

\section*{Author Biography}
\textbf{Xiaocheng Li} received his Ph.D. in 2020 from the Department of Management Science and Engineering at Stanford University. He will join the Department of Analytics, Marketing and Operations at Imperial College Business School from 2021. His recent work develops theories and algorithms for online learning and sequential decision-making problems, using tools from stochastics, optimization, and statistics. \\
\textbf{Huaiyang Zhong} received his Ph.D. from the Department of Management Science and Engineering at Stanford University in 2020.  He is currently a postdoctoral fellow at Harvard Medical School and Massachusetts General Hospital. His research focuses on developing and applying theories and models to solve medical decision making problems at both individual and population levels. \\
\textbf{Margaret L. Brandeau} is the Coleman F. Fung Professor in the School of Engineering and professor (by courtesy) of medicine at Stanford University. Her recent work has examined HIV and drug abuse prevention and treatment programs, programs to control the opioid epidemic, and preparedness plans for public health emergencies. 
%
%

\newpage
\bibliographystyle{ormsv080}
\bibliography{sample.bib}

\newpage

\end{document}